\documentclass[sn-mathphys-num]{sn-jnl}

\usepackage{graphicx}%
\usepackage{graphics}
\usepackage{epstopdf}
\usepackage{multirow}%
\usepackage{amsmath,amssymb,amsfonts}%
\usepackage{amsthm}%
\usepackage{mathrsfs}%
\usepackage[title]{appendix}%
\usepackage{xcolor}%
\usepackage{textcomp}%
\usepackage{manyfoot}%
\usepackage{booktabs}%
\usepackage{algorithm}%
\usepackage{algorithmicx}%
\usepackage{algpseudocode}%
\usepackage{listings}%
\usepackage{framed}
\usepackage{subcaption}
\usepackage{enumitem}
\usepackage{chngcntr}

\theoremstyle{thmstyleone}%
\newtheorem{theorem}{Theorem}
\newtheorem{proposition}[theorem]{Proposition}%
\newtheorem{corollary}{Corollary}[section]

\theoremstyle{thmstyletwo}%

\theoremstyle{thmstylethree}%
\newtheorem{definition}{Definition}%
\DeclareMathOperator*{\diag}{Diag}
\def\K{{\mathcal K}}
\counterwithin{equation}{section}

\begin{document}

\title[Iterative Reweighted Framework Based Algorithms for Sparse Linear Regression with  Generalized Elastic Net Penalty]{Iterative Reweighted Framework Based Algorithms for Sparse Linear Regression with  Generalized Elastic-Net Penalty}


\author[1]{\fnm{Yanyun} \sur{Ding}}\email{dingyanyun@szpu.edu.cn}
\author[2]{\fnm{Zhenghua} \sur{Yao}}\email{zhyao2000@163.com}
\author[2,3]{\fnm{Peili} \sur{Li}}\email{lipeili@henu.edu.cn}
\author*[3]{\fnm{Yunhai} \sur{Xiao}}\email{yhxiao@henu.edu.cn}

\affil[1]{\orgdiv{Institute of Applied Mathematics}, \orgname{Shenzhen Polytechnic University}, \orgaddress{\city{Shenzhen}, \postcode{518055}, \country{P. R. China}}}

\affil[2]{\orgdiv{School of Mathematics and Statistics}, \orgname{Henan University}, \orgaddress{\city{Kaifeng}, \postcode{475000}, \country{P.R. China}}}

\affil[3]{\orgdiv{Center for Applied Mathematics of Henan Province}, \orgname{Henan University}, \orgaddress{ \city{Zhengzhou}, \postcode{450046},  \country{P.R. China}}}


\abstract{The elastic net penalty is frequently employed in high-dimensional statistics for parameter regression and variable selection.
It is particularly beneficial compared to lasso when the number of predictors greatly surpasses the number of observations.
However, empirical evidence has shown that the $\ell_q$-norm penalty (where $0 < q < 1$) often provides better regression compared to the $\ell_1$-norm penalty, demonstrating enhanced robustness in various scenarios.
In this paper, we explore a generalized elastic net penalized model that employs a $\ell_r$-norm (where $r \geq 1$) in loss function to accommodate various types of noise, and employs a $\ell_q$-norm (where $0 < q < 1$) to replace the $\ell_1$-norm in elastic net penalty.
We theoretically prove that the local minimizer of the proposed model is a generalized first-order stationary point, and then derive the computable lower bounds for the nonzero entries of this generalized stationary point.
For the implementation, we utilize an iterative reweighted framework that leverages the locally Lipschitz continuous $\epsilon$-approximation, and subsequently propose two optimization algorithms. The first algorithm employs an alternating direction method of multipliers (ADMM), while the second utilizes a proximal majorization-minimization method (PMM), where the subproblems are addressed using the semismooth Newton method (SNN).
We also perform extensive numerical experiments with both simulated and real data, showing that both  algorithms demonstrate superior performance.
Notably, the PMM-SSN is efficient than ADMM, even though the latter provides a simpler implementation.}

\keywords{Sparse linear regression, Generalized elastic net penalty, Iterative reweighted framework,  Alternating direction method of multipliers, Proximal majorization-minimization method}



\maketitle

\section{Introduction}\label{intro}

In this paper, we mainly consider a high-dimensional sparse linear regression model. Suppose that the data set has $n$ observations with $p$ predictors. Let $y\in \mathbb{R}^n$ be the response and
$X:=[X_1,X_2,...,X_p]\in\mathbb{R}^{n\times p}$ be the sample matrix, where $X_i\in\mathbb{R}^n$ for $i=1,...,p$  are predictors. We assume that the following linear relation is data generation process:
$$
  y= X\underline\beta+\varepsilon,
$$
where $\underline\beta\in\mathbb{R}^p$ is the model fitting procedure produces the vector of coefficient, and $\varepsilon\in \mathbb{R}^n$ is the noise, which is not limited to a specific type of noise. Without loss of generality, we assume that the above data has been centralized.  In high-dimensional setting, the number of predictors $p$ is much larger than the number of observations $n$ and that $\underline\beta$ is known to be sparse a priori.
Under this assumption, penalty techniques are often employed to manage overfitting and facilitate variable selection.
A typical method called Lasso \cite{lasso2,lasso} adds a $\ell_1$-norm penalty to the ordinary least square loss function, which penalizes the absolute size of the coefficients. Subsequently, the estimation model is described as following:
$$
\min_{\beta\in \mathbb{R}^p} \frac{1}{2}\|X\beta-y\|^2_2+\lambda\|\beta\|_1,
$$
where $\|\beta\|_1$ is $\ell_1$-norm (named Lasso) and $\lambda>0$ is a tuning parameter. In view of the good characteristics of the $\ell_1$-norm, the Lasso does both continuous shrinkage and automatic variable selection simultaneously. Although Lasso has received extensive attention in different application fields, it has some limitations. For the case of $p\gg n $ and grouped variables situation, the Lasso can only select at most $n$ variables before it saturates, and it lacks the ability to reveal the grouping information \cite{net1,net3,net2}.
As a result, Zou et al \cite{net2} proposed an elastic-net penalty approach for both estimation and variable selection:
\begin{equation}\label{net0}
  \min_{\beta\in \mathbb{R}^p} \frac{1}{2}\|X\beta-y\|^2_2+\lambda_2\|\beta\|^2_2+\lambda_1\|\beta\|_1,
\end{equation}
where $\lambda_1$ and $\lambda_2$  are two non-negative tuning parameters. Similar to Lasso, the elastic-net penalty simultaneously does automatic variable selection and continuous shrinkage, and it can select groups of correlated variables \cite{net2}.

In fact, the $\ell_1$-norm is a loose approximation  function which leads to an over-regularized problem. Motivated by the fact that the $\ell_q$-norm penalty with $0 < q < 1$ can usually induce better sparse regression results than $\ell_1$-norm penalty \cite{lq1}, with improved robustness \cite{lq2}.
A natural improvement is the using of a $\ell_q$-norm  to replace the $\ell_1$-norm in elastic-net penalty, see also \cite{lqnet}:
\begin{equation}\label{net1}
  \min_{\beta\in \mathbb{R}^p} \frac{1}{2}\|X\beta-y\|^2_2+\lambda_2\|\beta\|^2_2+\lambda_1\|\beta\|^q_q,
\end{equation}
where $\|\beta\|_q=(\sum_{i=1}^p|\beta_i|^q)^{1/q}$ is a $\ell_q$-norm.
It is important to note that the error vector $\varepsilon$ may not approach zero and may not follow a normal distribution. Although least squares estimation is computationally attractive, it relies on knowing the standard deviation of Gaussian noise, which can be challenging to estimate in high-dimensional settings.
A key question is whether we can extend the model (\ref{net1}) to accommodate different types of noise, beyond the least squares loss function.

This question is important because when noise does not follow a normal distribution, alternative loss functions may be more robust.
The least absolute deviation loss (a.k.s. $\ell_1$-norm) is commonly used for scenarios with heavy-tailed or heterogeneous noise, often resulting in better regression performance than least squares loss \cite{l11,l12,l13,l14,l15}.
The square-root loss \cite{l21} removes the need to know or pre-estimate the deviation and can achieve near-oracle rates of convergence under some appropriate conditions \cite{l22,l23,l24}.
It also has been shown that $\ell_{\infty}$-norm loss performs better when deal with the uniformly distributed noise and quantization error \cite{l33,l31,l32}.
Therefore, in this paper, we mainly concentrate on a generalized elastic-net problem with a $\ell_r$-norm (where $r\geq 1$) loss function to result in the following model:
\begin{equation}\label{lp}
	\min_{\beta\in \mathbb{R}^p} \Big\{ F(\beta):=\|X\beta-y\|_r+\lambda_2\|\beta\|^2_2+\lambda_1\|\beta\|^q_q\Big\},
\end{equation}
where $\lambda_1$ and $\lambda_2 \geq 0$ are tuning parameters.
We note that the  $\ell_r$-norm, associated with the parameter $r$, is assumed to have a strongly semismooth proximal mapping.
This assumption regarding $\ell_r$-norm is mild and often encountered. For instance, the proximal mappings associated with $\ell_1$-, $\ell_2$- and $\ell_\infty$-norms are all strongly semismooth.
These specific norms are commonly employed in optimization problems that involve loss functions for addressing various types of noise \cite{l33,l34}.
The selection of $r$ is suitable for different types of noise, for instance, $r = 1$ is ideal for heavy-tailed or heterogeneous noise, $r = 2$ is appropriate when the standard deviation of Gaussian noise is unknown, and $r = \infty$ fits scenarios involving uniformly distributed noise.

It is known that $ \ell_q$ for $0 < q < 1$ is nonsmooth and nonconvex, and it may not even be Lipschitz continuous \cite{np1,np2}.
To address this issue, several effective iterative reweighted algorithms regarding  $\ell_1$- and $\ell_2$-norms have been proposed \cite{ir1, ir3, ir2, lqnet}.
For instance, Lu \cite{lu} presented a Lipschitz continuous $\epsilon$-approximation of the $\ell_q$-norm and empirically demonstrated that this approach generally outperforms some iterative reweighted methods.
These algorithms rely on least absolute deviation or least squares approximations of the $\ell_q$-norm, with the goal of finding an approximate solution to (\ref{net1}).
The successful execution of these algorithms primarily arises from the smoothness of the least squares loss term in (\ref{net1}). However, the nonconvex nonsmooth properties of $\ell_q$-norm, combined with the nonsmooth nature of the $\ell_r$-norm loss function, make the numerical solving of model (\ref{lp}) considerably more difficult.

Drawing inspiration from the work of Lu \cite{lu}, this paper transforms the generalized elastic-net model \eqref{lp} into a convex yet nonsmooth iterative reweighted $\ell_1$-norm minimization by utilizing an $\epsilon$-approximation.
In the context, we begin by defining the generalized first-order stationary point of (\ref{lp}), which helps us establish lower bounds and identify its local minimizers.
Subsequently, we demonstrate that any accumulation point obtained from the sequence generated by the approximation technique converges to a generalized first-order stationary point, provided that the approximation parameter $\epsilon$ remains below a certain threshold value.
It is crucial to highlight that the estimation model obtained from this approximation is convex but non-smooth.
As a result, we  develop two efficient practical methods to fully utilize the underlying structures of the convex approximate minimization model.
The first algorithm is the alternating direction method of multipliers (ADMM), known for its ease of implementation.
However, ADMM is still a first-order method, which may restrict its efficiency when tackling problems with high accuracy requirements.
In light of this, we also employ a novel proximal majorization minimization (PMM) approach, see \cite{l24}, to solve the subproblem through its dual using a highly efficient semismooth Newton (SSN) method.
Finally, we perform a series of numerical experiments that highlight the remarkable superiority of the generalized elastic-net model (\ref{lp}) and show that both of the proposed algorithms are   efficient.

The remainder of this paper is organized as follows. In Section \ref{preli}, we summarize some key concepts and address the $\epsilon$-approximation problem. Section \ref{comp} outlines our motivation and constructs the algorithms, including their convergence results. Section \ref{num4} presents numerical results to demonstrate the effectiveness of our approaches. Finally, we conclude this paper in Section \ref{con5}.
\section{Some theoretical properties}\label{preli}

\subsection{Preliminary results}\label{presult}

Let $\mathbb{R}^{p}$ be a $p$-dimensional Euclidean space endowed with an inner product $\langle \cdot,\cdot \rangle $ and its induced norm $\|\cdot\|_2$, respectively.
Given $\beta^* \in \mathbb{R}^p$, let $T = \{i : \beta_i^* \neq 0\}$ represents the support set of $\beta^*$, and let $\bar{T}$ denotes its complement within $\{1, \ldots, n\}$.
Let $f:\mathbb{R}^{p}\rightarrow(-\infty,+\infty]$ be a closed proper convex function.
The proximal mapping and the Moreau envelope function of $f$ with a parameter $t > 0$ are defined, respectively, as follows
$$
\operatorname{Prox}_{tf}(x) :=\underset{y \in \mathbb{R}^{p}}{\operatorname{argmin}}\Big\{f(y)+\frac{1}{2t}\|y-x\|_2^{2}\Big\}, \quad \forall x \in \mathbb{R}^{p},
$$
and
$$
\Phi_{t f}(x) :=\min _{y \in \mathbb{R}^{p}}\Big\{f(y)+\frac{1}{2 t}\|y-x\|_2^{2}\Big\}, \quad \forall x \in \mathbb{R}^{p}.
$$
In particular, the explicit forms of the proximal operator for the $\ell_r$-norm with $r=1$, $2$, and $\infty$ may refer to \cite{l33,smooth}. Additionally, it is noted in \cite{m1,m2} that $\Phi_{t f}(x)$ is continuously differentiable and convex, with its gradient given by
$$
\nabla \Phi_{t f}(x)=t^{-1}(x-\operatorname{Prox}_{t f}(x)), \quad \forall x \in \mathbb{R}^{p}.
$$
By the Moreau's identity theorem \cite[Theorem 35.1]{Rock35}, it holds that
$$
\operatorname{Prox}_{tf}(x)+t\operatorname{Prox}_{f^{*}/t}(x/t)=x,
$$
where $f^*$ represents the Fenchel conjugate function of $f$.

At the end of this subsection, we  review some results concerning the proximal mappings for several common types of vector norms.
Let $f(x):=t\|x\|_1$ with a $t>0$, then $f^{*}(x)=\delta_{B^{(t)}_{\infty}}(x)$ where ${B^{(t)}_{\infty}}:=\{x \ | \ \|x\|_{\infty}\leq t\}$ and
$$
\operatorname{Prox}_f(x)=x-\Pi_{B^{(t)}_{\infty}}(x) \quad \text{with} \quad
(\Pi_{B^{(t)}_{\infty}}(x))_i=\left\{\begin{array}{ll}
	x_i, & \text { if }  |x_i|\leq t,\\
	\operatorname{Sign}(x_i) t, & \text { if }  |x_i|>t,
\end{array}\right.
$$
where  `$\operatorname{Sign}(\cdot)$' is a sign function of a vector.
Let $f(x):=t\|x\|_2$, then $f^{*}(x)=\delta_{B^{(t)}_{2}}(x)$ where ${B^{(t)}_{2}}:=\{x \ | \ \|x\|_{2}\leq t\}$ and
$$
\operatorname{Prox}_f(x)=x-\Pi_{B^{(t)}_{2}}(x)\quad \text{with} \quad
\Pi_{B^{(t)}_{2}}(x)=\left\{\begin{array}{ll}
	x, & \text { if }  \|x\|_2\leq t,\\
	t\frac{x}{\|x\|_2}, & \text { if }  \|x\|_2>t.
\end{array}\right.
$$
Let $f(x):=t\|x\|_{\infty}$, then $f^{*}(x)=\delta_{B^{(t)}_{1}}(x)$ where ${B^{(t)}_{1}}:=\{x\ | \ \|x\|_{1}\leq t\}$ and
$$
\operatorname{Prox}_f(x)=x-\Pi_{B^{(t)}_{1}}(x) \quad \text{with} \quad
\Pi_{B^{(t)}_{1}}(x)=\left\{\begin{array}{ll}
	x^*, & \text { if }  \|x\|_1\leq t,\\
	\mu P_{x} \Pi_{\Delta_{n}}\left(P_{x} x /t\right), & \text { if }  \|x\|_1>t,
\end{array}\right.
$$
where $P_{x}:=\diag(\operatorname{Sign}(x))$ and $\Pi_{\Delta_{n}}(\cdot)$ denotes the projection onto the simplex $\Delta_{n}:=\{x \in \mathbb{R}^{n} \mid e_{n}^{\top} x=1, x \geq 0\}$, in which $\diag(\cdot)$ denotes a diagonal matrix with elements of a vector on its diagonal positions.
\subsection{Lower bound for nonzero entries of  generalized  stationary point}
We now discuss the lower bound for nonzero entries of the generalized first-order stationary point of the model (\ref{lp}).
This discussion draws inspiration from the works referenced in \cite{lu,Xiu}; however, the subsequent process represents an enhancement through the use of a nonsmooth
$\ell_r$-norm loss function.

At the first place, we give the definition of a generalized first-order stationary point.
\begin{definition}
	Suppose that $\beta^*$ is a vector in $\mathbb{R}^p$ and that $\Lambda^*=\operatorname{Diag}(\beta^*)$. We say that $\beta^*$ is a generalized first-order stationary point of (\ref{lp}) if
	\begin{equation}\label{fot}
		0\in \Lambda^*\Big(\partial \|X\beta^*-y\|_r+2\lambda_2 \beta^*\Big)+\lambda_1 q\big|\beta^*\big|^q.
	\end{equation}
	where $\big|\beta^*\big|^q:=(\big|\beta_1^*\big|^q,\ldots,\big|\beta_p^*\big|^q)^\top$.
\end{definition}

At the second place, we show that every local minimizer of (\ref{lp}) is indeed a generalized first-order stationary point.

\begin{theorem}\label{pro1}
	Let $\beta^*$ be a local minimizer of (\ref{lp}). Then (\ref{fot}) holds, that is to say,  $\beta^*$ is a generalized first-order stationary point.
\end{theorem}
\begin{proof}
	Suppose that $\beta^*$ is a local minimizer of (\ref{lp}). Evidently, $\beta^*$ also serves as a local minimizer for
	\begin{equation}\label{lpt}
		\min_{\beta\in\mathbb{R}^p}\Big\{\|X\beta-y\|_r+\lambda_2\|\beta\|^2_2+\lambda_1\|\beta\|^q_q \ | \ \beta_i=0, \ i \neq T\Big\}.
	\end{equation}
	It is important to highlight that the objective function of (\ref{lpt}) is locally  Lipschitz continuous at $\beta^*$ within the subspace indexed by $T$.
	In the case of $i\in T$, using the first-order optimality condition, we have
	\begin{equation}\label{foo}
		0\in \Big(\partial\|X\beta^*-y\|_r\Big)_i+2\lambda_2 \beta_i^*+\lambda_1 q\big|\beta_i^*\big|^{q-1}\operatorname{Sign}(\beta_i^*).
	\end{equation}
	Multiplying $\beta_i^*$ on both sides of (\ref{foo}),  it gets that
	$$
	0\in \Big(\partial \|X\beta^*-y\|_r\Big)_i\beta_i^*+2\lambda_2 (\beta_i^*)^2+\lambda_1 q\big|\beta_i^*\big|^{q}.
	$$
	Since $\beta^*_i = 0$ for $i \neq T$,  the above relation is also valid for $i \neq T$. Therefore, (\ref{fot}) holds.
\end{proof}

At the third place, we establish the lower bound for the nonzero entries of the generalized first-order stationary points, or equivalently, the local minimizers of (\ref{lp}).

\begin{theorem}\label{the1}
	Suppose that $\beta^*$ is a generalized first-order stationary point of (\ref{lp}). Then for $r=1$, $2$, and $\infty$, it holds that
	\begin{equation}\label{lower}
		\big|\beta^*_i\big|>\Big(\frac{q\lambda_1   }{\left\|X_{i}\right\|_{r}}\Big)^{\frac{1}{1-q}}, \quad\forall \ i \in T,
	\end{equation}
where $X_i$ is the $i$-th column of $X$.
\end{theorem}
\begin{proof}
	From the definition of the generalized first-order stationary point, we get for $i \in T $ that
	$$
	0\in \Big(\partial\|X\beta^*-y\|_r\Big)_i+2\lambda_2 \beta_i^*+q\lambda_1 \big|\beta_i^*\big|^{q-1}\operatorname{Sign}(\beta_i^*),
	$$
    or equivalently,
	$$
	-2\lambda_2 \beta_i^*-q\lambda_1 \big|\beta_i^*\big|^{q-1}\operatorname{Sign}(\beta_i^*)\in \Big(\partial\|X\beta^*-y\|_r\Big)_i.
	$$
   It is easy to see that $\operatorname{Sign}(\beta_i^*)\neq0$ and $|\operatorname{Sign}(\beta_i^*)|=1$ in the case of $i\in T$.  Then, multiplying `$\operatorname{Sign}(\beta_i^*)$' on both sides of the above relation, it yields
$$
-2\lambda_2 |\beta_i^*|-q\lambda_1 \big|\beta_i^*\big|^{q-1}\in \operatorname{Sign}(\beta_i^*)\Big(\partial\|X\beta^*-y\|_r\Big)_i.
$$
For $r=1$, $2$, and $\infty$, using the definition of $\partial\|X\beta^*-y\|_r$, we get
\begin{equation}\label{defp}
	\Big|\operatorname{Sign}(\beta_i^*)(\partial\|X\beta^*-y\big\|_r)_i\Big|\leq\big\|X_i\big\|_r,
\end{equation}
which indicates $2\lambda_2 |\beta_i^*|+q\lambda_1 |\beta_i^*|^{q-1}\leq\|X_i\|_{r}$. Hence, it is not difficult to conclude that
$$
q\lambda_1 |\beta_i^*|^{q-1}<\|X_i\|_{r},
$$
which means that $|\beta_i^*|>\left(\frac{q\lambda_1   }{\left\|X_{i}\right\|_{r}}\right)^{\frac{1}{1-q}}$ for any $i\in T$.
\end{proof}

\subsection{Generalized stationary point of $\epsilon$-approximation problem}

The $\ell_q$-norm is non-Lipschitz continuity at certain points, including those with zero components.
As a result, its Clarke subdifferential is not well-defined, complicating the development of algorithms.
Drawing the inspiration from Lu  \cite{lu}, we propose replacing $\|\beta\|_q$-term with an $\epsilon$-approximation function
$\sum_{i=1}^{p} h_{u_{\epsilon}}\left(\beta_{i}\right)$, in which
$$
h_{u_{\epsilon}}(\beta_{i}):=\min_{0\leq s\leq u_{\epsilon}}q\Big(|\beta_{i}|s-\frac{q-1}{q}s^{\frac{q}{q-1}}\Big),
$$
where $u_{\epsilon}:=\left({\epsilon}/(p\lambda_1 )\right)^{\frac{q-1}{q}}$ and $\epsilon>0$ is a small scalar.
We observe that the Clarke subdifferential of $h_{u_{\epsilon}}$, denoted by $\partial h_{u_{\epsilon}}$, exists everywhere and is given by
\begin{equation}\label{cs}
	\partial h_{u_{\epsilon}}\left(\beta_{i}\right)=\left\{\begin{array}{ll}
		q|\beta_i|^{q-1} \operatorname{Sign}\left(\beta_{i}\right), & \text { if }\left|\beta_{i}\right|>u_{\epsilon}^{\frac{1}{q-1}}, \\[2mm]
		q u_{\epsilon} \operatorname{Sign}\left(\beta_{i}\right), & \text { if }\left|\beta_{i}\right| \leq u_{\epsilon}^{\frac{1}{q-1}}.
	\end{array}\right.
\end{equation}
This subdifferential is used in the subsequent algorithmic developments and theoretical analysis. For more detail, one may refer to \cite{lu}.

Using the $\epsilon$-approximation strategy, the problem (\ref{lp}) transforms into
\begin{equation}\label{model2}
	\min_{\beta\in\mathbb{R}^{p}}\Big\{F_{(\epsilon)}(\beta):=\|X\beta-y\|_r+\lambda_2\|\beta\|_2^2+\lambda_1\sum_{i=1}^{p}h_{u_{\epsilon}}(\beta_{i})\Big\}.
\end{equation}
Next we  show that the generalized stationary point of the corresponding $\epsilon$-approximation problem (\ref{model2}) is also the one of   (\ref{lp}).

\begin{theorem}\label{the2}
	Suppose that $\beta^*$ is a generalized first-order stationary point of (\ref{model2}). Assume that $\epsilon>0$ be a small constant satisfies
	\begin{equation}\label{bound}
		0<\epsilon<p\lambda_1 \left(\frac{\left\|X_{i}\right\|_{r}}{q\lambda_1 }\right)^{\frac{q}{q-1}}.
	\end{equation}
	Then, $\beta^*$ is also a generalized first-order stationary point of (\ref{lp}). Furthermore, the nonzero entries of $\beta^*$ meet the lower bound condition given in (\ref{lower}).
\end{theorem}

\begin{proof}
	Using the fact that $\beta^*$ is a generalized first-order stationary point of (\ref{model2}), we get
	$0\in F_{\epsilon}(\beta^*).$ On the one hand, for any $ i\in T$, it gets that
	\begin{equation}\label{eq1}
		0\in \Big(\partial \|X\beta^*-y\|_r\Big)_i+2\lambda_2 \beta_i^*+\lambda_1\partial h_{u_{\epsilon}}\left(\beta^*_{i}\right),
	\end{equation}
which implies
$$
-2\lambda_2 \beta_i^*-\lambda_1 \partial h_{u_{\epsilon}}(\beta^*_{i})\in \Big(\partial \|X\beta^*-y\|_r\Big)_i.
$$
Multiplying $\operatorname{Sign}(\beta_i^*)$ on both sides, it gets that
$$
-2\lambda_2 |\beta_i^*|-\lambda_1  \operatorname{Sign}(\beta_i^*)\partial h_{u_{\epsilon}}\left(\beta^*_{i}\right)\in \operatorname{Sign}(\beta_i^*)\Big(\partial \|X\beta^*-y\|_r\Big)_i.
$$
Combining with (\ref{cs}) and noting the definition of $u_{\epsilon}$, we get
$$
-2\lambda_2 |\beta_i^*|-\lambda_1  |\partial h_{u_{\epsilon}}\left(\beta^*_{i}\right)|\in \operatorname{Sign}(\beta_i^*)\Big(\partial \|X\beta^*-y\|_r\Big)_i.
$$
From (\ref{defp}), it is not hard to deduce that
$$
2\lambda_2 |\beta_i^*|+\lambda_1  |\partial h_{u_{\epsilon}}\left(\beta^*_{i}\right)|\leq\|X_i\|_r.
$$
Notice that $2\lambda_2 |\beta_i^*|\geq0$, then we have
$$
\lambda_1  |\partial h_{u_{\epsilon}}\left(\beta^*_{i}\right)|\leq\|X_i\|_r,
$$
or equivalently,
\begin{equation}\label{bound2}
	|\partial h_{u_{\epsilon}}\left(\beta^*_{i}\right)|\leq\|X_i\|_r/\lambda_1.
\end{equation}
Next, we demonstrate that $|\beta_i^*|>u_{\epsilon}^{\frac{1}{q-1}}$ for all $i\in T$.
Assume that $0<|\beta_i^*|\leq u_{\epsilon}^{\frac{1}{q-1}}$ for some $i\in T$,  we get $|\partial h_{u_{\epsilon}}\left(\beta^*_{i}\right)|=qu_{\epsilon}$ from (\ref{cs}).
Recalling the definition of $u_{\epsilon}$ and noting the inequality (\ref{bound}), it is easy to know that
$$
|\partial h_{u_{\epsilon}}\left(\beta^*_{i}\right)|
=qu_{\epsilon}=
q\Big(\frac{\epsilon}{p\lambda_1}\Big)^{\frac{q-1}{q}}>\frac{\|X_i\|_r}{\lambda_1}.
$$
This conclusion contradicts (\ref{bound2}), leading to the result that $|\beta_i^*| > u_{\epsilon}^{\frac{1}{q-1}}$ for all $i \in T $.
From (\ref{cs}),  it yields that
$$
\partial h_{u_{\epsilon}}\left(\beta^*_{i}\right)=
q |\beta^*_i|^{q-1} \operatorname{Sign}\left(\beta^*_{i}\right).
$$
Therefore (\ref{eq1}) can be expressed as
$$
0\in \Big(\partial \|X\beta^*-y\|_r\Big)_i+2\lambda_2 \beta_i^*+q\lambda_1 |\beta^*_i|^{q-1} \operatorname{Sign}\left(\beta^*_{i}\right).
$$
Then, multiplying $\beta_i^*$ on both sides, it gets
$$
0\in \Big(\partial \|X\beta^*-y\|_r\Big)_i\beta_i^*+2\lambda_2 (\beta_i^*)^2+\lambda_1 q|\beta_i^*|^{q}.
$$
which means
$$
0\in \Lambda^*\Big(\partial \|X\beta^*-y\|_r+2\lambda_2 \beta^*\Big)+\lambda_1 q|\beta^*|^q.
$$
On the other hand, given that $\beta^*_i = 0$ for $i \notin T $, it is clear that the aforementioned relation also holds for $i \notin T $. Consequently, equation (\ref{fot}) is satisfied, which means $\beta^*$ is a generalized first-order stationary point of (\ref{lp}). Subsequently, the second part of this theorem is followed from Theorem \ref{the1}.
\end{proof}

The following corollary demonstrates that the minimizer of (\ref{model2}) is also a generalized first-order stationary point of (\ref{lp}), which directly follows from Theorems \ref{the2} and \ref{pro1}.
\begin{corollary}\label{cor11}
	Suppose that $\beta^*$ is an optimal solution of (\ref{lp}), then $\beta^*$ is also an optimal solution of (\ref{model2}) provided that the parameter $\epsilon>0$ satisfies \eqref{bound}.   Moreover, the nonzero entries of $\beta^*$ satisfy the lower bound given in (\ref{lower}).
\end{corollary}
\section{ADMM  and PMM-SSN algorithms}\label{comp}
\setcounter{equation}{0}

This section focuses on developing algorithms to address the proposed model (\ref{lp}) through an iterative reweighted framework. Let $\beta^k\in\mathbb{R}^p$ be a given point, and let $w^k\in\mathbb{R}^p$ with its $i$-th entry $w_{i}^k:=\min\{u_\epsilon,|\beta^k_i|^{q-1}\}$ and $u_\epsilon=(\frac{\epsilon}{p\lambda_1 })^{\frac{q-1}{q}}$.

We now consider the following iterative reweighted minimization problem
\begin{equation}\label{ir}
	\min_{\beta\in \mathbb{R}^p}\Big\{F_{\beta^k}(\beta):=\|X\beta-y\|_r+\lambda_2\|\beta\|_2^2+q\lambda_1 \|W_k\beta\|_1 \Big\},
\end{equation}
where $W_{k}:=\operatorname{Diag}(w^k)$, i.e., a diagonal matrix with component $ w_i^k$ at its $i$-th position.
From an initial point $\beta^0$, the iterative reweighted framework  generates a sequence $\{\beta^k\}$ according to solving a series of the following reweighted minimization problem
\begin{equation}\label{iter1}
	\beta^{k+1}:= \arg\min_{\beta\in \mathbb{R}^p} F_{\beta^k}(\beta).
\end{equation}
It is evident that this is a convex minimization problem, because the nonconvex $\ell_q$-term with $0<q<1$ has been eliminated.

The following theorem asserts that any accumulation point of the sequence $\{\beta^k\}$ generated by (\ref{iter1}) is a generalized first-order stationary point of (\ref{lp}). The proof process of the following theorem can be viewed as an extension of \cite[Theorem 3.1]{lu}, particularly in the context of employing an $\ell_r$-norm loss function.

\begin{theorem}\label{the31}
	Suppose that $\{\beta^k\}$ be the sequence generated by (\ref{iter1}) and that $\bar{\beta}$ be an accumulation point of $\{\beta^k\}$. Assume that $\epsilon>0$ satisfies (\ref{bound}).  Then $\bar{\beta}$ is a generalized first-order stationary point of (\ref{lp}), i.e., (\ref{fot}) holds at $\bar{\beta}$. Moreover, the nonzero entries of $\bar{\beta}$ satisfy the lower bound (\ref{lower}).
\end{theorem}
\begin{proof}
	At the beginning, we define
	$$
	F(\beta,w):=\|X\beta-y\|_r+\lambda_2\|\beta\|^2_2+q\lambda_1\sum_{i=1}^p\left(|\beta_i|w_i-\frac{q-1}{q}	w_i^{\frac{q}{q-1}} \right).
	$$
	It is a trivial task to see that $w^k$ is actually a minimizer of $F(\beta^k,w)$ with   constraint $w\in[0,u_\epsilon]$, that is,
	\begin{equation}\label{sk}
		w^k=\arg\min_{0\leq w\leq u_{\epsilon}}F(\beta^k,w).
	\end{equation}
	And that $\beta^{k+1}$ is one of a minimizer of $F(\beta,w^k)$, that is,
	$$
	\beta^{k+1}\in \arg\min_{\beta\in\mathbb{R}^p} F(\beta,w^k).
	$$
By comparing the expressions for $F_{(\epsilon)}(\beta)$ in (\ref{model2}) and $F(\beta,w)$, we can get from (\ref{sk}) that
$$
F_{(\epsilon)}(\beta)=\min_{0\leq \omega\leq u_{\epsilon}}F(\beta,\omega), \quad \text{and}\quad F_{(\epsilon)}(\beta^k)=F(\beta^k,\omega^k).
$$
Then we obtain that
\begin{equation}\label{FG}
	F_{(\epsilon)}(\beta^{k+1})=F(\beta^{k+1},\omega^{k+1})\leq F(\beta^{k+1},\omega^{k})\leq F(\beta^{k},\omega^{k})=F_{(\epsilon)}(\beta^{k}),
\end{equation}
which indicates that $\{F_{(\epsilon)}(\beta^{k})\}$ is a non-increasing sequence.
Let $\bar\beta$ be an accumulation point of the subsequence $\{\beta^k\}_\K$ such that $\{\beta^k\}_\K\rightarrow \bar\beta$. It follows that $F_{(\epsilon)}(\beta^{k})\rightarrow F_{(\epsilon)}(\bar\beta)$ due to the monotonicity and continuity of $F_{(\epsilon)}(\cdot)$.
Let $\bar w_i=\min\{u_{\epsilon},|\bar\beta_i|^{q-1}\}$ for all $i$.
We then observe that $\{w^k\}_\K\rightarrow \bar w$ and that $F_{(\epsilon)}(\bar\beta)=F(\bar\beta,\bar\omega)$.
Using (\ref{FG}) and that $F_{(\epsilon)}(\beta^{k})\rightarrow F_{(\epsilon)}(\bar\beta)$, we see that $F(\beta^{k+1},\omega^k)\rightarrow F_{(\epsilon)}(\bar\beta)=F(\bar\beta,\bar w)$. Besides, combining with (\ref{sk}), it holds that
$$
F(\beta,w^k)\geq F(\beta^{k+1},w^k), \quad \forall \ \beta\in \mathbb{R}^p.
$$
Let $k$ $(\in \K) \rightarrow \infty$, we obtain that
$$
F(\beta,\bar w)\geq F(\bar\beta,\bar w), \quad \forall \ \beta\in \mathbb{R}^p,
$$
which yields that
\begin{equation}\label{Gmin}
	\bar\beta\in \arg\min_{\beta\in \mathbb{R}^p}\left\{\|X\beta-y\|_r+\lambda_2\|\beta\|^2_2+\lambda_1q\sum_{i=1}^p|\beta_i|\bar w_i\right\}.
\end{equation}
From the first-order optimality condition of (\ref{Gmin}), we have
\begin{equation}\label{opt}
	0\in \Big(\partial \|X\bar\beta-y\|_r\Big)_i+2\lambda_2 \bar\beta_i+q\lambda_1\bar{w_i}\operatorname{Sign}(\bar\beta_i), \quad \forall i.
\end{equation}
It follows from the definition of $w_i^k$, we can express $\bar w_i^k$ as
$$
\bar w_{i}=\left\{\begin{array}{ll}
	\left|\bar\beta_{i} \right|^{q-1}, & \text { if }\left|\bar\beta_{i}\right|>u_{\epsilon}^{\frac{1}{q-1}}, \\
	u_{\epsilon}, & \text { if }\left|\bar\beta_{i}\right| \leq u_{\epsilon}^{\frac{1}{q-1}},
\end{array}\right.
$$
Then, it shows that (\ref{opt}) can be expressed as
$$
0\in \Big(\partial \|X\bar\beta-y\|_r\Big)_i+2\lambda_2 \bar\beta_i+\lambda_1\partial h_{u_{\epsilon}}(\bar\beta_i), \quad \forall \ i.
$$
In light of the above analysis, it follows that $\bar\beta$ is a generalized first-order stationary point of $F_{(\epsilon)}(\cdot)$ defined in \eqref{model2}. Using the above results and Theorem \ref{the2}, we conclude that $\bar\beta$ is a generalized first-order stationary point of (\ref{lp}). Thereafter, the nonzero entries of $\bar{\beta}$ satisfy the lower bound (\ref{lower}) from Theorem \ref{the1}.
\end{proof}

Observing that problem (\ref{ir}) is convex and nonsmooth, we  can develop two algorithms that effectively leverage the underlying its structures.

\subsection{ADMM for (\ref{ir})}
In this section, we aim to utilize a well-known first-order method called ADMM to solve (\ref{ir}).
For this purpose, we let $\eta:=X\beta-y$ and $\theta:=W_{k}\beta$, then (\ref{ir}) is equivalent to
\begin{equation}\label{admm0}
	\begin{array}{lll}
		\min\limits_{\beta,\eta,\theta} & \|\eta\|_r+\lambda_2\|\beta\|_2^2+\lambda_1q\|\theta\|_1\\[2mm]
		\text{s.t.} & X\beta-y-\eta=0,\\[2mm]
		&W_{k}\beta-\theta=0.
	\end{array}
\end{equation}
Let $\sigma>0$ be a penalty parameter, the augmented Lagrangian function associated with problem (\ref{admm0})  is given by
\begin{align*}
	\mathcal{L}_{\sigma}(\beta,\eta,\theta;u,v)&= \|\eta\|_r +\lambda_2\|\beta\|_2^2+\lambda_1q\|\theta\|_1+\langle u, X\beta-y-\eta\rangle+\frac{\sigma}{2}\|X\beta-y-\eta\|^2_2\\[2mm]
	&+\langle v, W_{k}\beta-\theta\rangle+\frac{\sigma}{2}\|W_{k}\beta-\theta\|^2_2,
\end{align*}
where $u\in\mathbb{R}^{n}$ and $v\in\mathbb{R}^{p}$ are  multipliers associated with the constraints.
When the traditional ADMM is employed, it miminizes the augmented Lagrangian function alternatively with respect to two non-overlapping blocks, such as
one block $\beta$ and the other block $(\eta,\theta)$. This leads to the following iterative scheme from an initial point $(\beta^{(0)},\eta^{(0)},\theta^{(0)})$:
$$
\left\{
\begin{array}{l}
	\beta^{(j)}:=\arg\min_{\beta}\mathcal{L}_{\sigma}(\beta,\eta^{(j)},\theta^{(j)};u^{(j)},v^{(j)}),
	\\[3mm]
	(\eta^{(j+1)}, \theta^{(j+1)}):=\arg\min_{\eta, \theta }\mathcal{L}_{\sigma}(\beta^{(j+1)},\eta,\theta;u^{(j)},v^{(j)}),\\[3mm]
	u^{(j+1)}:=u^{j}+\tau\sigma\big(X\beta^{(j+1)}-\eta^{(j+1)}-y\big),\\[3mm]
	v^{(j+1)}:=v^{j}+\tau\sigma\big(W_{k}\beta^{(j+1)}-\theta^{(j+1)}\big),
\end{array}
\right.
$$
where $\tau$ is steplength chosen within the interval $(0,(1+\sqrt{5})/2)$.
Clearly, the computational cost primarily arises from solving the subproblems related to the variables $\beta$ and $(\eta, \theta)$. We now demonstrate that each subproblem admits closed-form solutions by utilizing the proximal mappings of the $\ell_r$-norm for $r = 1$, $2$, and  $\infty$.

We   now concentrate on addressing both subproblems involved in this iterative scheme.
Firstly, with given $(\eta^{(j)},\theta^{(j)}$ and multipliers $(u^{(j)},v^{(j)})$, it is easy to deduce that solving the $\beta$-subproblem is actually equivalent to finding a solution of the following linear system:
$$
(2\lambda_2I_p+\sigma X^{\top}X+\sigma W_{k}^{\top}W_{k})\beta=
\sigma X^{\top}(y+\eta^{(j)}-\sigma^{-1}u^{(j)})+\sigma W_{k}^{\top}(\theta^{(j)}-\sigma^{-1}v^{(j)}),
$$
where $I_p$ is an identity matrix in $p$-dimensional real space.
Fortunately, solving this linear system is not particularly challenging because the well-known Sherman-Morrison-Woodbury formula can be used to obtain the inverse of its coefficient matrix.
Secondly, when we focus on the augmented Lagrangian function with variables $(\eta, \theta)$, we observe that both variables are independent. This implies that solving the $(\eta, \theta)$-subproblem together is the same as solving each one individually.
As a result, it is a trivial task to deduce that the $\eta$- and $\theta$- subproblems can be written as the following proximal mapping forms
$$
\eta^{(j+1)}=\operatorname{Prox}_{{\sigma}^{-1}\|\cdot\|_r}\Big(X\beta^{(j+1)}-y+u^{(j)}/\sigma\Big),
$$
and
$$
\theta^{(j+1)}=\operatorname{Prox}_{{\sigma}^{-1}\lambda_1q\|\cdot\|_1}\Big(W_{k}\beta^{(j+1)}+v^{(j)}/\sigma\Big).
$$
From subsection \ref{presult}, we know that the proximal mappings of the $\ell_r$-norm function for the case of $r = 1$, $2$, and  $\infty$ have analytical expressions.
This suggests that the iterative framework can be easily implemented.

Based on the above analysis, we outline the iterative steps of ADMM as it is applied to solve the convex subproblem (\ref{ir}).
\begin{framed}
\noindent
{\bf Algorithm 1: ADMM for (\ref{ir})}
\vskip 1.0mm \hrule \vskip 1mm
\noindent
\begin{itemize}[leftmargin=10mm]
\item[Input.] With given $\beta^{(0)}:=\beta^k$, choose an initial point $(\eta^{(0)},\theta^{(0)};u^{(0)},v^{(0)})$.
Choose positive constants $\sigma>0$ and $\tau\in(0,(1+\sqrt{5})/2)$. For $j=0,1,\ldots$, do the following operations iteratively:
\item[Step 1.] Given  $\eta^{(j)}$, $\theta^{(j)}$, and $u^{(j)}$, $v^{(j)}$.
Employ a numerical method to compute an  solution $\beta^{(j+1)}$ to the linear system
\begin{align*}
\Big(2\lambda_2I_p+\sigma X^{\top}X+\sigma W_{k}^{\top}W_{k}\Big)\beta
 =&
\sigma X^{\top}\Big(y+\eta^{(j)}-\sigma^{-1}u^{(j)}\Big)\\
+&\sigma W_{k}^{\top}\Big(\theta^{(j)}-\sigma^{-1}v^{(j)}\Big).
\end{align*}
\item[Step 2.] Given $\beta^{(j+1)}$ and $u^{(j)}$. Compute $\eta^{(j+1)}$ according to
$$
\eta^{(j+1)}:=\operatorname{Prox}_{{\sigma}^{-1}\|\cdot\|_r}\Big(X\beta^{(j+1)}-y+u^{(j)}/\sigma\Big),
$$
\item[Step 3.] Given $\beta^{(j+1)}$ and $v^{(j)}$. Compute $\theta^{(j+1)}$ according to
$$
\theta^{(j+1)}:=\operatorname{Prox}_{{\sigma}^{-1}q\lambda_1\|\cdot\|_1}\Big(W_{k}\beta^{(j+1)}+v^{(j)}/\sigma\Big).
$$
\item[Step 4.] Using $\beta^{(j+1)}$, $\eta^{(j+1)}$, $\theta^{(j+1)}$, update $u^{(j+1)}$ and $v^{(j+1)}$ according to
$$
u^{(j+1)}:=u^{(j)}+\tau\sigma\big(X\beta^{(j+1)}-\eta^{(j+1)}-y\big),
$$
and
$$v^{(j+1)}:=v^{(j)}+\tau\sigma\big(W_{k}\beta^{(j+1)}-\theta^{(j+1)}\big).$$
\end{itemize}
\end{framed}

The convergence results for the ADMM algorithm can be directly referenced from \cite{chen,MF}. Hence, we omit the convergence theorem here.
While ADMM is straightforward to implement, it is only a first-order algorithm, which may not yield higher precision solutions quickly. In the next section, we will leverage the specific structure of problem (\ref{ir}) and employ the highly efficient SSN method, utilizing the PMM approach through its dual.
\subsection{SSN based on PMM for (\ref{ir})}
The function $F_{\beta^k}(\beta)$ obtained from (\ref{ir}) is convex but nonsmooth. However, it is well-known that the SSN method requires a smooth function to achieve a faster local convergence rate. Therefore, we can instead employ the PMM framework and then take a dual approach, see also in \cite{l24}.

Let $\mu > 0$ be a given positive scalar that can also be determined dynamically. We consider (\ref{ir}) with an additional proximal point term:
\begin{equation}\label{irpmm}
	\beta^{k+1}:=\arg\min_{\beta\in \mathbb{R}^p}\Big\{\tilde{F}_{\beta^k}(\beta):=\|X\beta-y\|_r+\lambda_2\|\beta\|_2^2+q\lambda_1 \|W_{k}\beta\|_1+\frac{\mu}{2}\|X\beta-X{\beta}^k\|^2_2 \Big\},
\end{equation}
where $\beta^k\in\mathbb{R}^p$ is a current point.
It should be noted that the last term is actually a precondition used to derive a dual problem with favorable structures.

To derive its dual, it is convenient for us to express (\ref{irpmm}) as:
\begin{equation}\label{subp2}
	\begin{array}{ll}
		\min\limits_{\beta,\eta} & \|\eta\|_r+\lambda_2\|\beta\|_2^2+q\lambda_1\|W_{k}\beta\|_1
		+\frac{\mu}{2}\|\eta+y-X\beta^k\|^2_2\\[3mm]
		\text{s.t.} & X\beta-\eta=y.\\
	\end{array}
\end{equation}
The Lagrangian function associated with problem (\ref{subp2})  is given by
\begin{align*}
	\mathcal{L}(\beta,\eta;u)&= \|\eta\|_r+\lambda_2 \|\beta\|_2^2+q\lambda_1\|W_{k}\beta\|_1
	+\frac{\mu}{2}\|\eta+y-X\beta^k\|_2^2+\langle u, X\beta-y-\eta\rangle,
\end{align*}
where $u\in\mathbb{R}^m$ is a multiplier associated with the constraint.
The Lagrangian dual function $D(u)$ is defined as the minimum value of the Lagrangian function over $(\beta, \eta)$, that is
\begin{align*}
	D(u)
	=& \inf_{\beta,\eta} \mathcal{L}(\beta,\eta;u)\\[1mm]
	=& \min_{\beta} \Big\{\lambda_2 \|\beta\|_2^2+\langle u, X\beta \rangle+q\lambda_1\|W_{k}\beta\|_1\Big\}\\[1mm]
	+& \min_{\eta} \Big\{\|\eta\|_{r}+\frac{\mu}{2}\|\eta+y-X{\beta}^k\|_2^2-\langle u,\eta\rangle\Big\}-\langle u,y\rangle.
\end{align*}
Let $g_{k}(\beta) = q\lambda_1 \|W_{k} \beta\|_1 $. One can use the Moreau envelope function given in subsection \ref{presult} to express the minimize values of the $\beta$- and $\eta$-subproblems, that is,
\begin{align}
	\mathcal{X}(u):= & \min_{\beta} \Big\{\lambda_2 \|\beta\|_2^2+\langle u, X\beta \rangle+q\lambda_1\|W_{k}\beta\|_1\Big\}\nonumber\\[1mm]
	=&\Phi_{(2\lambda_2)^{-1}g_k(\cdot)}\Big(-{(2\lambda_2)}^{-1}X^{\top}u\Big)-\frac{1}{4\lambda_2}\|X^{\top}u\|^2_2,\label{betapro}
\end{align}
and
\begin{align*}
	\mathcal{Y}(u):=& \min_{\eta} \Big\{\|\eta\|_{r}+\frac{\mu}{2}\|\eta+y-X{\beta}^k\|_2^2-\langle u,\eta\rangle\Big\}\\[1mm]
	=&\Phi_{\mu^{-1} \| \cdot \|_{r}}\Big( \mu^{-1}u-y+X{\beta}^k \Big)-\frac{\mu}{2}\|\mu^{-1}u-y+X{\beta}^k\|^2_2.
\end{align*}
Here, the minimizers of $\beta$- and $\eta$-subproblems  arise from the expression for the proximal mapping:
$$
	\beta=\operatorname{Prox}_{(2\lambda_2)^{-1}g_k(\cdot)}\Big(-{(2\lambda_2)}^{-1}X^{\top}u\Big),
$$
and
$$
\eta=\operatorname{Prox}_{\mu^{-1} \| \cdot \|_{r}}\Big( \mu^{-1}u-y+X{\beta}^k \Big),
$$
provided that $u$ is given.

The Lagrangian dual problem of (\ref{subp2}) involves maximizing the dual function $D(u)$, which can be equivalently expressed as the following optimization problem:
\begin{equation}\label{duall}
\min_{u} \Big\{\Theta(u):=\langle u,y\rangle-\mathcal{X}(u)-\mathcal{Y}(u)\Big\}.
\end{equation}
It is established in \cite{m1,m2} that $\Theta(u)$ is convex and first-order continuously differentiable with its gradient given by:
$$
\nabla\Theta(u)=-X\text{Prox}_{(2\lambda_2)^{-1}g_k(\cdot)}\big(-{(2\lambda_2)}^{-1}X^{\top}u\big)+\text{Prox}_{\mu^{-1} \| \cdot \|_{r}}\big(\mu^{-1}u-y+X\beta^k \big)+y.
$$
As a result, the minimizer of $\Theta(u)$ can be obtained by solving the following first-order optimality condition:
$$
\nabla\Theta(u)=0.
$$
Additionally, it is important to note that the proximal mappings of the $\ell_r$-norm for $r=1$, $2$, and $\infty$ are strongly semismooth. Consequently, the above equations are also strongly semismooth, allowing for the application of the well-known semismooth Newton method.

It is noteworthy that the proximal operators $\text{Prox}_{\| \cdot \|_r}(\cdot)$ with $r=1$, $2$ and $\infty$ is nonsmooth, which means that the Hessian of $\Theta(u)$ is unavailable.
Fortunately, due to the strong semismoothness of the proximal mappings embedded in $\nabla\Theta(u)$,  the generalized Hessian of $\Theta(u)$  can be obtained explicitly from \cite{l33,smooth}.
Noting that the proximal mapping operators with $r=1$, $2$ and $\infty$ are Lipschitz continuous, then the following multifunction is well defined:
\begin{align*}
\tilde{\partial}^2\Theta(u):=&(2\lambda_2)^{-1}X\partial\Big(\text{Prox}_{(2\lambda_2)^{-1}g_{k}(\cdot)}\big(-{(2\lambda_2)}^{-1}X^{\top}u\big)X^{\top}\Big)\\
+&\mu^{-1} \partial\Big(\text{Prox}_{\mu^{-1} \| \cdot \|_{r}}\big( \mu^{-1}u-y+X\beta^k \big)\Big).
\end{align*}
Choose $U\in \partial\text{Prox}_{(2\lambda_2)^{-1}g_{k}(\cdot)}\big(-{(2\lambda_2)}^{-1}X^{\top}u\big)$ and $V\in \partial\text{Prox}_{\mu^{-1} \| \cdot \|_{r}}\big( \mu^{-1}u-y+X{\beta}^k \big)$. And let
$$
H:=(2\lambda_2)^{-1}XUX^{\top}+\mu^{-1} V.
$$
It is clear that $H$ is a positive semi-definite matrix, which implies that the SSN method still cannot be used. To overcome this issue, we modify $ H $ by adding a term:
$$
\tilde{H}:=H+\nu I_{n},
$$
where $\nu$ is a very small positive constant.

Based on the analysis above, we are ready to outline the iterative framework for the SSN method within the framework of PMM approach. The iterative steps are described as follows:

\begin{framed}
\noindent
{\bf Algorithm 2: PMM based on SNN for \eqref{irpmm} (PMM-SSN)}
\vskip 1.0mm \hrule \vskip 1mm
\begin{itemize}[leftmargin=10mm]
\item[Step 0.]  At current point $\beta^{k} \in \mathbb{R}^{p}$ choose an initial point $u^{(0)}\in\mathbb{R}^n$.
Input  $\mu>0$, $\varrho\in (0,1/2)$ and $\delta\in(0,1)$.
\item[Step 1.]  Applying SSN to find an approximate solution $\bar u^{k+1}$ such that
    $$
   \bar u^{k+1}:\approx\arg\min_{u}\Theta(u)
    $$
    via the following steps:
    \begin{itemize}
    \item [Step 1.1.]  Select $U^{(j)}\in \partial\text{Prox}_{(2\lambda_2)^{-1}g_k(\cdot)}\big(-{(2\lambda_2)}^{-1}X^{\top}u^{(j)}\big)$ and $V^{(j)}\in \partial\text{Prox}_{\mu{-1} \| \cdot \|_{r}}\big( \mu^{-1}u^{(j)}-y+X{\beta}^k \big)$, and set
    $$
    \tilde{H}^{(j)}:=(2\lambda_2)^{-1}XU^{(j)}X^{\top}+\mu^{-1} V^{(j)}+\nu I_{n}.
    $$
    Employ a numerical method (e.g., conjugate gradient  method) to compute an approximate solution $\Delta u^j$ to the linear system
\[
  \tilde{H}^{(j)}\Delta u^{(j)}+\nabla \Theta (u^{(j)})=0.
\]
\item[Step 1.2.] Find $\alpha_j:=\delta^{t_j}$, where $t_j$ is the first nonnegative integer $t$ such that
\[
\Theta (u^{(j)}+\delta^{t}\Delta u^{(j)})\leq \Theta (u^{(j)})+\varrho \delta^{t}\langle \nabla \Theta (u^{(j)}), \Delta u^{(j)}\rangle.
\]
\item[Step 1.3.] Set $u^{(j+1)}:=u^{(j)}+\alpha_j\Delta u^{(j)}$. Let $j=j+1$, go to Step 1.1.
    \end{itemize}
\item[Step 2.]  Compute
$$
\beta^{k+1}:=\text{Prox}_{(2\lambda_2)^{-1}g_k(\cdot)}\Big( -{2\lambda_2}^{-1}X^{\top}\bar u^{(k+1)}\Big).
$$
\end{itemize}
\end{framed}

We will now establish some relationships between the problems (\ref{irpmm}) and (\ref{ir}) under certain conditions. First, we demonstrate that $\tilde{F}_{\beta^k}(\beta)$ converges to $F_{\beta^k}(\beta)$ when the positive number $\mu$ is sufficiently small.

\begin{theorem}\label{pro3}
	The optimal objective value of problem (\ref{irpmm}) converges to that of  (\ref{ir})  if $\mu \downarrow 0$, that is,
	$$
	\lim_{\mu\downarrow 0}\min_{\beta\in \mathbb{R}^p}\tilde{F}_{\beta^k}(\beta)=\min_{\beta\in \mathbb{R}^p} F_{\beta^k}(\beta).
	$$
\end{theorem}
\begin{proof}
	It holds that
	$$
	\min_{\beta\in \mathbb{R}^p}\tilde{F}_{\beta^k}(\beta)\geq\min_{\beta\in \mathbb{R}^p} F_{\beta^k}(\beta).
	$$
	For any $\mu>0$ and $\beta\in\mathbb{R}^p$, we have that
	$$\min_{\beta\in \mathbb{R}^p}\tilde{F}_{\beta^k}(\beta)\leq \|X\beta-y\|_r+\lambda_2\|\beta\|_2^2+q\lambda_1 \|W_{k}\beta\|_1
	+\frac{\mu}{2}\|X\beta-X\beta^k\|^2_2.
	$$
	Taking limits on both hand-sides of this inequality as $\mu\downarrow 0$, it gets that
	$$
	\lim_{\mu\downarrow 0}\min_{\beta\in \mathbb{R}^p}\tilde{F}_{\beta^k}(x)\leq \|X\beta-y\|_r+\lambda_2\|\beta\|_2^2+q\lambda_1 \|W_{k}\beta\|_1,
	$$
	which means $$\lim_{{\mu\downarrow 0}}\min_{\beta\in \mathbb{R}^p}\tilde{F}_{\beta^k}(\beta)\leq\min_{\beta\in \mathbb{R}^p} F_{\beta^k}(\beta),$$ that is
	$$
	\lim_{\mu\downarrow 0}\min_{\beta\in \mathbb{R}^p}\tilde{F}_{\beta^k}(\beta)=\min_{\beta\in \mathbb{R}^p} F_{\beta^k}(\beta),
	$$
	which shows the conclusion is true.
\end{proof}

Next, we  show that the optimal solution $\bar\beta$ of $\tilde F_{\beta^k}(\beta)$ is actually the minimizer of $F_{\beta^k}(\beta)$ in the case of $\beta^k\equiv \bar\beta$.
\begin{proposition}\label{pro4}
	The vector $\bar\beta$ is an optimal solution of \eqref{ir} if and only if there exists a $\mu \geq 0$ such that
	$$
	\bar\beta \in \arg\min_{\beta\in\mathbb{R}^p} \tilde{F}_{\bar\beta}(\beta).
	$$
	This means that the optimal solution to the problem $\min_{\beta\in\mathbb{R}^p} \tilde{F}_{\bar\beta}(\beta)$ is indeed the same as the optimal solution to $\min_{\beta\in\mathbb{R}^p} F_{\bar\beta}(\beta)$. Consequently, the nonzero entries of $\bar\beta$ satisfy the lower bound (\ref{lower}) when $\epsilon>0$ satisfies (\ref{bound}).
\end{proposition}

\begin{proof}
	Noting that $F_{\bar\beta}(\cdot)$ is locally Lipschitz continuous near $\bar\beta$ and convex at $\bar\beta $, we have that $ 0 \in \partial F_{\bar\beta}(\beta) $ is equivalent to $\bar\beta $ being an optimal solution of $ F_{\bar\beta}(\cdot) $.
		Additionally, since the function $\tilde{F}_{\bar\beta}(\cdot) $ is convex, it follows that $0 \in \partial \tilde{F}_{\bar\beta}(\beta) $ is equivalent to $ \bar\beta \in \arg\min_{\beta\in\mathbb{R}^p} {\tilde{F}_{\bar\beta}(\beta)} $.
		Combining these results with the fact that $\partial F_{\bar\beta}(\bar\beta) = \partial \tilde{F}_{\bar\beta}(\bar\beta) $, we conclude that the optimal solution $\bar\beta $ for $\min_{\beta\in \mathbb{R}^p} \tilde{F}_{\bar{\beta}}(\beta) $ is also the optimal solution for $ F_{\bar{\beta}}(\beta) $.
		Subsequently, applying Theorem \ref{the31} and \ref{the1}, we can directly obtain the remaining result of this theorem.
\end{proof}
\section{Numerical experiments}\label{num4}

In this section, we showcase several simulated and real-world examples to highlight the performance of the proposed algorithms in tackling non-convex and non-Lipschitz penalized high-dimensional linear regression problems.
All the experiments are performed with Microsoft Windows 10 and MATLAB R2021b, and run on a PC with an Intel Core i5 CPU at 2.70 GHz and 16 GB of memory.

\subsection{Experiments setting}\label{set41}

{\bf{(i) Data generations:}} In the simulation experiments, we generate the $n\times p$ matrix $X$ whose rows are drawn independently from $\mathcal{N}(0,\Sigma)$ with $\sum_{i-j}=\kappa^{|i-j|}$ and $1\leq i,j\leq p$, where $\kappa$  is the correlation coefficient of matrix $X$.
In order to generate the target regression coefficient $\underline\beta\in \mathbb{R}^{p}$, we randomly select a subset of $\{1,...,p\}$ to form the active set $\mathcal{A}^*$ with $|\mathcal{A}^*|=K<m$. Let $R=m_2/m_1$, where $m_2=\|\beta^*_{\mathcal{A}^*}\|_{\max}$ and $m_1=\|\beta^*_{\mathcal{A}^*}\|_{\min}=1$. Here $\|\cdot\|_{\max}$ and $\|\cdot\|_{\min}$ denote the maximum and minimum absolute values of a vector, respectively. Then the $K$ nonzero coefficients in $\beta^*$ are uniformly distributed in $[m_1,m_2]$. The response variable is derived by $y=X\beta^*+\alpha\varepsilon$, where $\alpha$ is a noise level, and $\varepsilon$ is a noise, e.g., log-normal noise, Gaussian noise, and uniform noise.\\
{\bf{(ii) Parameters settings:}}  The values of the parameters involved in model \eqref{lp} and the algorithms ADMM and PMM-SSN will be given in the following context.
Additionally, through extensive tests, we observed that the choice of the initial point has minimal impact on each algorithm. As a result, we set the initial point uniformly as $0$.\\
{\bf{(iii) Stopping criteria:}} We use the following  relative residual denoted as $\eta_1$ and $\eta_2$ to measure the accuracy the approximate optimal solution:
$$
\eta_1:=\frac{\|\beta_k-{\text{Prox}_{\lambda_1\|\cdot\|_1}}(\beta_k-X^{\top}(X\beta_k-y))\|_2}{1+\|\beta_k\|_2+\|X^{\top}(X\beta_k-y)\|_2}
$$
and
\[
\eta_2:=\frac{\|\beta^{k+1}-\beta^k\|_2}{\max\{\|\beta^k\|_2,1\}}.
\]
In addition, we stop the iterative process if  $\eta_2\leq 1e-4$ or the iteration number achieves $2000$.\\
{\bf{(iv) Evaluation indexes:}} To highlight the efficiency and accuracy of PMM-SSN and ADMM, we perform experiments comparing these two methods with each other and with other algorithms.
Specifically, we consider four evaluation measurements based on multiple independent experiments: the average CPU time (Time) in seconds, the standard deviation (SD), the average $\ell_2$-norm relative error, and the average mean squared error (MSE), defined as follows:
$$
\text{RE}:=\sum\frac{\|{\bar\beta}-\underline\beta\|_2}{\|\underline\beta\|_2}/N,\quad\text{and}\quad
	\text {MSE} :=\sum\frac{\|\bar{\beta}-\underline\beta\|_2}{p}/N,
$$
where $\bar{\beta}$ is the estimated regression coefficient by ADMM or PMM-SSN, and $N$ is the total times of tests.
Clearly, a smaller `Time' indicates faster calculation speed, while lower values of `SD', `RE', and `MSE' reflect higher solution quality.

\subsection{Test on simulated data}
In this part, we do a series of numerical experiments on fixed and random data to numerically evaluate the superiority of our propose algorithms.
Meanwhile, we also conduct performance comparisons among these algorithms to address the proposed problem (\ref{lp}).

\subsubsection{Test on a fixed data}
In this subsection, we analyze the numerical behavior of PMM-SSN and ADMM algorithms based
on $100$ independent tests.
In each test, we use the following parameters' values: $n=300$, $p=1000$, $K=10$, $R=100$, and $N=5$.

We investigate the performance of PMM-SSN and ADMM for variable selection and parameter estimation.
In this test, we generate the coefficient matrix $X$ with $\kappa=0.2$ and set the true regression parameter as $\underline\beta$, whose $10$ non-zero elements are $ \underline\beta_{30}=-5$, $
\underline\beta_{198}=2$, $\underline\beta_{269}=-3$, $\underline\beta_{395}=-2$, $\underline\beta_{442}=4$, $\underline\beta_{495}=-4$, $\underline\beta_{637}=1$, $\underline\beta_{776}=5$, $\underline\beta_{777}=-1$, and $\underline\beta_{865}=3$.
We set the noise level as $\alpha=1e-3$. For simplicity, we set $q=1$ here, which represents the elastic-net penalty.
To clearly see the numerical performance of PMM-SSN and ADMM algorithms, we show the box plots in Figure \ref{fig1}.
Additionally, we compute the overall average standard deviation (SD) to illustrate the recovery performance of both algorithms at zero input positions.
It is important to note that, to help readers assess the bias, the values are presented to four decimal places. Furthermore, to evaluate the estimation accuracy for a large number of zero elements, we also provide the average relative error (RE) values.

In this test, we choose $\sigma=1e-3$ and $\tau=1.618$ in ADMM, while, we set $\mu=1e-3$, $\rho=0.9$, $\varrho=0.1$, $\delta=0.5$ and $\nu=1e-6$ in PMM-SSN.
On the model parameters, we choose different values based on the various types of noise.
In the case of log-normal distribution noise (LN), we set $r=1$, i.e., the loss function is $\|Ax-b\|_1$.
In the case of Gaussian distribution noise (GN), we set $r=2$, i.e., the loss function is $\|Ax-b\|_2$.
The tuning parameters' values in both cases are chosen as $\lambda_1=0.02$ and $\lambda_2=5e-4$.
In the case of uniform distribution noise (UN), we set $r=\infty$, i.e., the loss function is $\|Ax-b\|_{\infty}$. Meanwhile, the weighting parameters are chosen as $\lambda_1=0.02$ and $\lambda_2=5e-6$.

Under these parameters' settings, for the LN case, the PMM-SSN yields a RE value of $2.69e-2$, while the ADMM yields a RE value of $3.2e-2$. The corresponding SD values are $8.9e-3$ for PMM-SSN and $8.8e-3$ for ADMM.
For the case of GN, the PMM-SSN achieves RE$=2.600e-2$ and SD$=8.900e-3$, while the ADMM yields RE$=3.190e-2$ and SD$=8.800e-3$.
For the case of UN, the RE values obtained by PMM-SSN and ADMM are $1.200e-3$  and $1.200e-3$, respectively, while the SD values are $4.8170e-4$ and $2.297e-4$.
The results show that the performance of PMM-SSN and ADMM are comparable.

\begin{figure}[htbp]
	\centering
	\begin{subfigure}[b]{0.32\textwidth}
		\includegraphics[width=\textwidth]{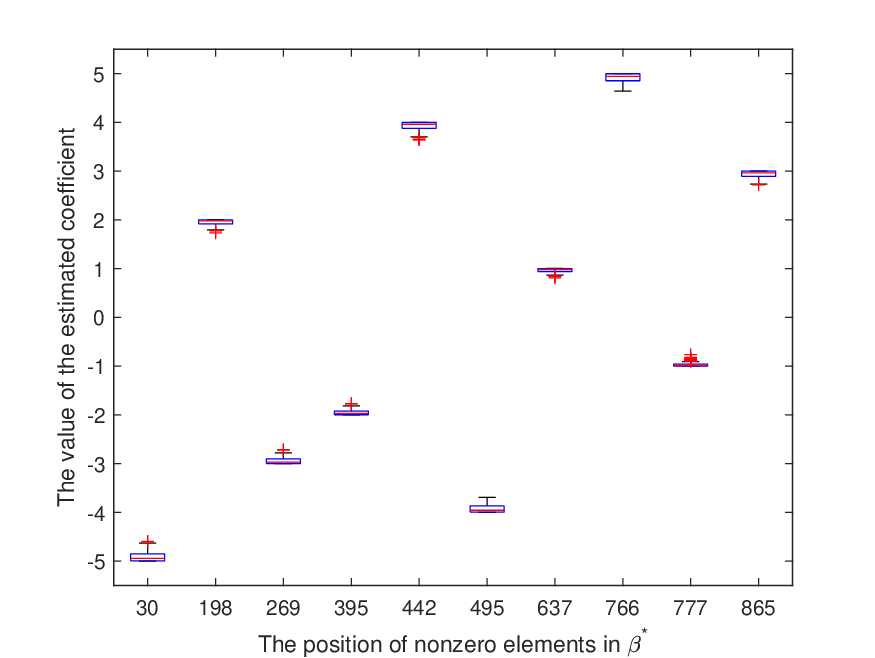}
		\caption{LN (PMM-SSN)}
	\end{subfigure}
	\begin{subfigure}[b]{0.32\textwidth}
		\includegraphics[width=\textwidth]{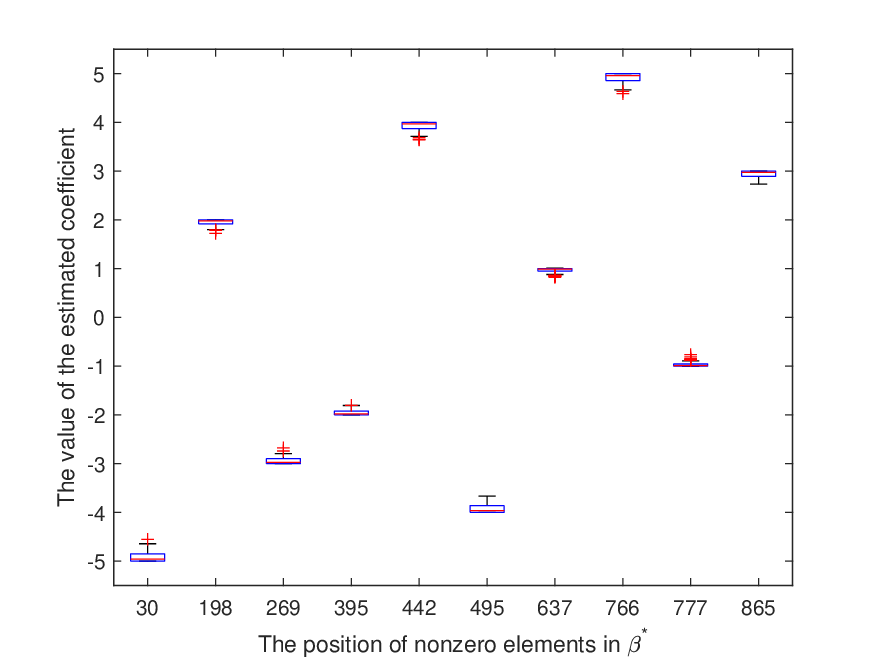}
		\caption{GN (PMM-SSN)}
	\end{subfigure}
	\begin{subfigure}[b]{0.32\textwidth}
		\includegraphics[width=\textwidth]{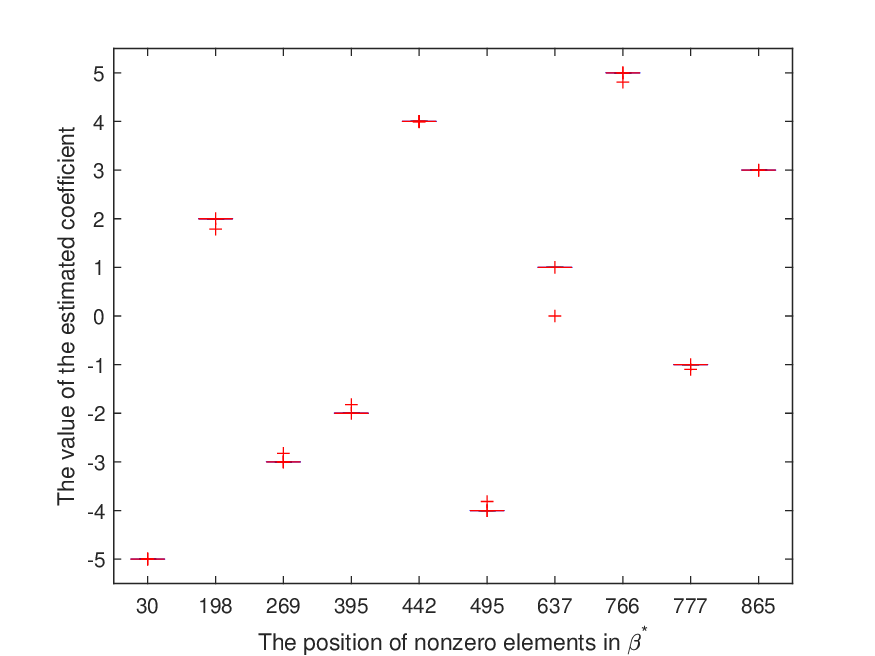}
		\caption{UN (PMM-SSN)}
	\end{subfigure}
	\begin{subfigure}[b]{0.32\textwidth}
		\includegraphics[width=\textwidth]{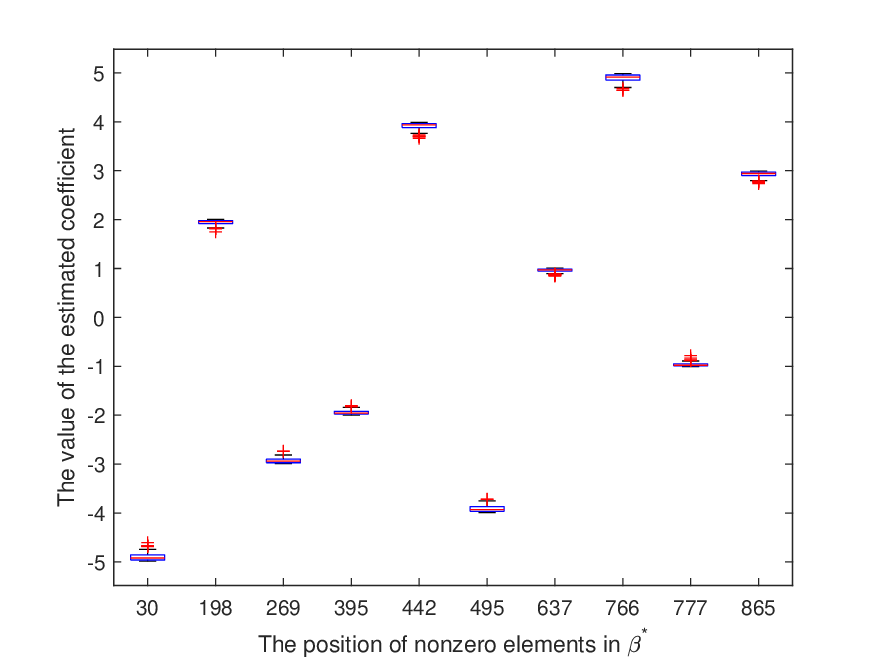}
		\caption{LN (ADMM)}
	\end{subfigure}
	\begin{subfigure}[b]{0.32\textwidth}
		\includegraphics[width=\textwidth]{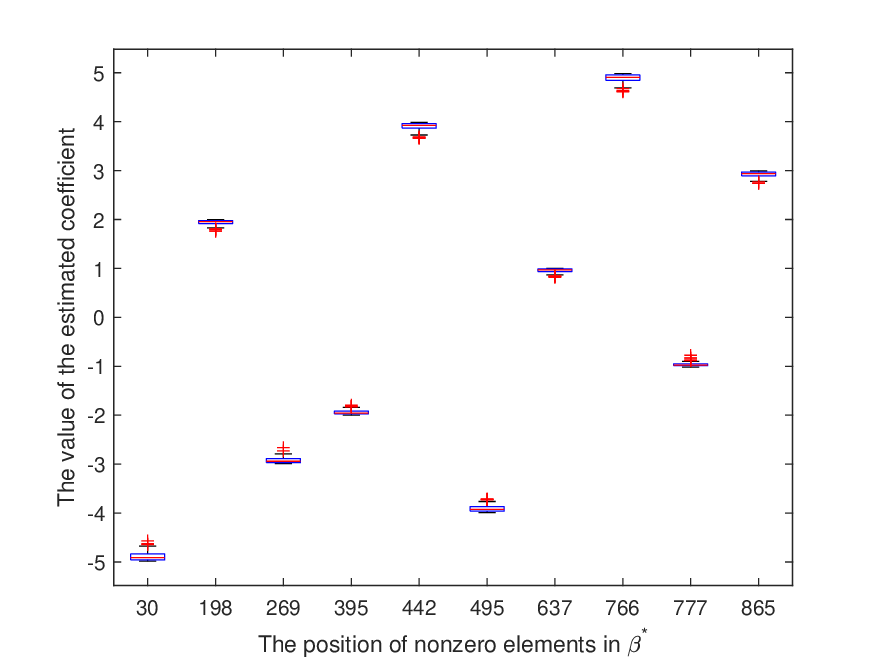}
		\caption{GN (ADMM)}
	\end{subfigure}
	\begin{subfigure}[b]{0.32\textwidth}
		\includegraphics[width=\textwidth]{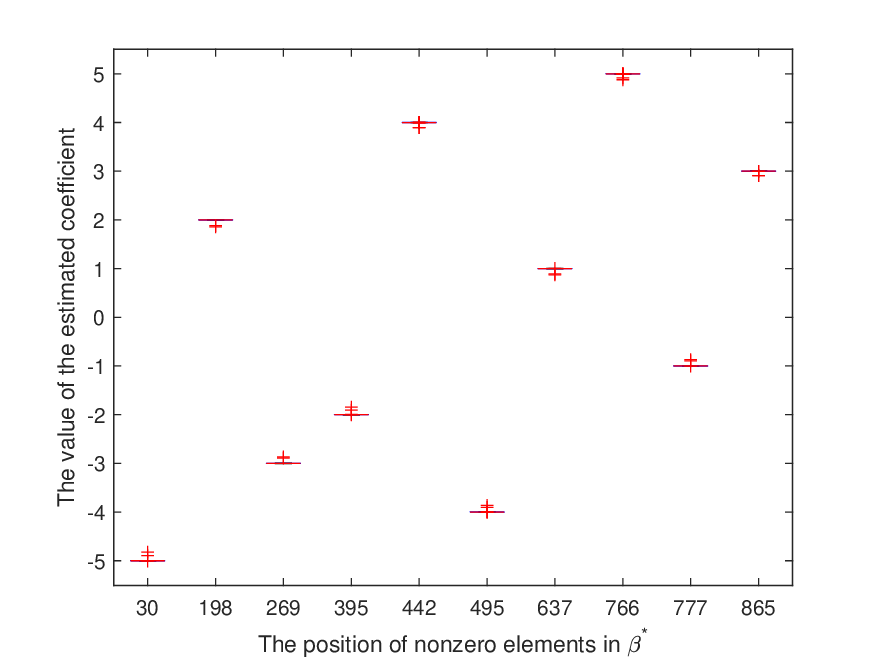}
		\caption{UN (ADMM)}
	\end{subfigure}
	\vspace{-.0cm}
	\caption{\small Distribution of predicted points at nonzero input locations}
	\label{fig1}
	\vspace{-.2cm}
\end{figure}

Figure \ref{fig1} illustrates the distribution of the regression data generated by PMM-SSN and ADMM at the non-zero input positions.
From this figure, we can observe that both algorithms exhibit stability in addressing linear regression problems across different noise types, as evidenced by the average value being closely aligned with the true non-zero elements.
This phenomenon shows that PMM-SSN and ADMM  can estimate non-zero elements with a high accuracy.

\subsubsection{Test with different parameter's values}

In this part, we numerically compare the  performance of PMM-SSN and ADMM from the perspective of the RE values using random simulated data.
Based on the different types of noise, we report the average number of iterations with different sparsity levels in Figure \ref{fig2}.
In this test, we use the approach described in Section \ref{set41} to generate the data except that the correlation coefficient $\kappa$ is set as $0.3$.
In this simulation, we set the $\ell_q$-norm with $q=[0.2,0.3,0.5,0.7,0.9]$ in model \eqref{lp}.
On the tuning parameters,  we chose $\lambda_1=0.2$, $\lambda_2=1e-4$, and$\mu=1e-2$ in the case of LN and GN noise.
But, in the case of UN noise, we choose $\lambda_1=0.03$ and $\lambda_2=1e-7$.
In addition, we also test both algorithms using different sparsity levels  $K=[5,\ldots,15]$.
The results are presented in Figure \ref{fig2}, with noise types LN, GN, and UN arranged from left to right.

\begin{figure}[htbp]
	\centering
	\begin{subfigure}[b]{0.32\textwidth}
		\includegraphics[width=\textwidth]{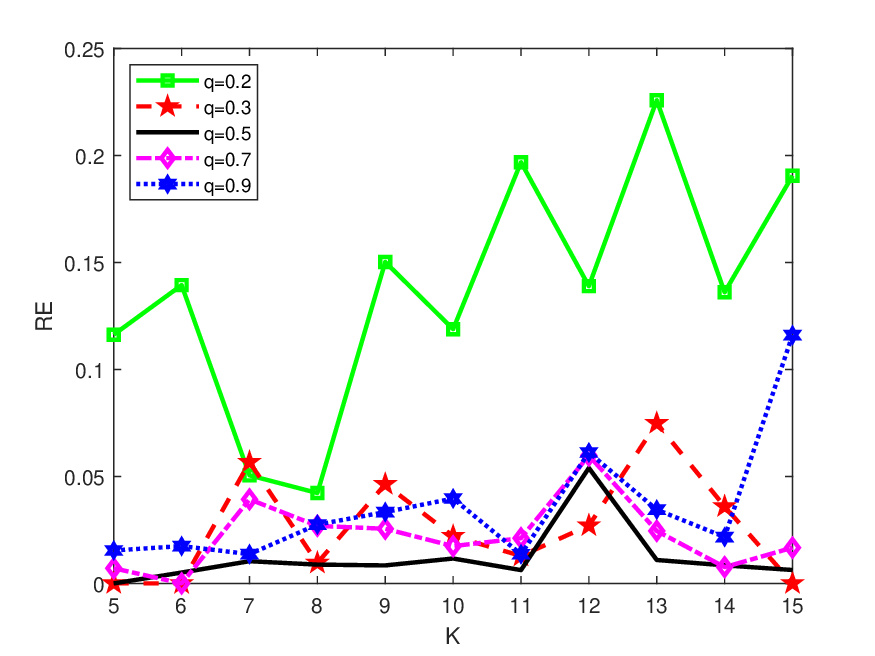}
		\caption{LN (SSN-PMM)}
	\end{subfigure}
	\begin{subfigure}[b]{0.32\textwidth}
		\includegraphics[width=\textwidth]{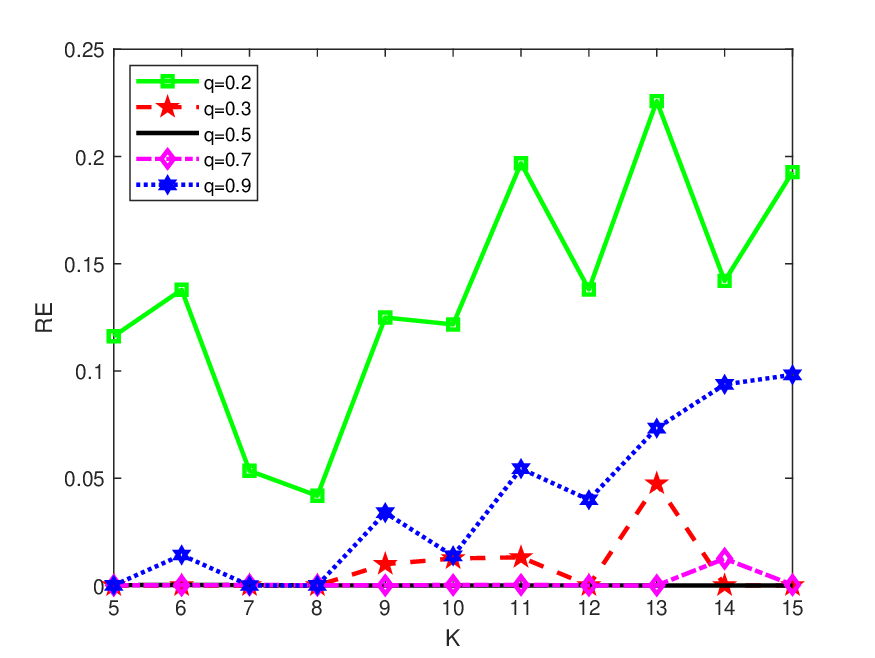}
		\caption{GN (SSN-PMM)}
	\end{subfigure}
	\begin{subfigure}[b]{0.32\textwidth}
		\includegraphics[width=\textwidth]{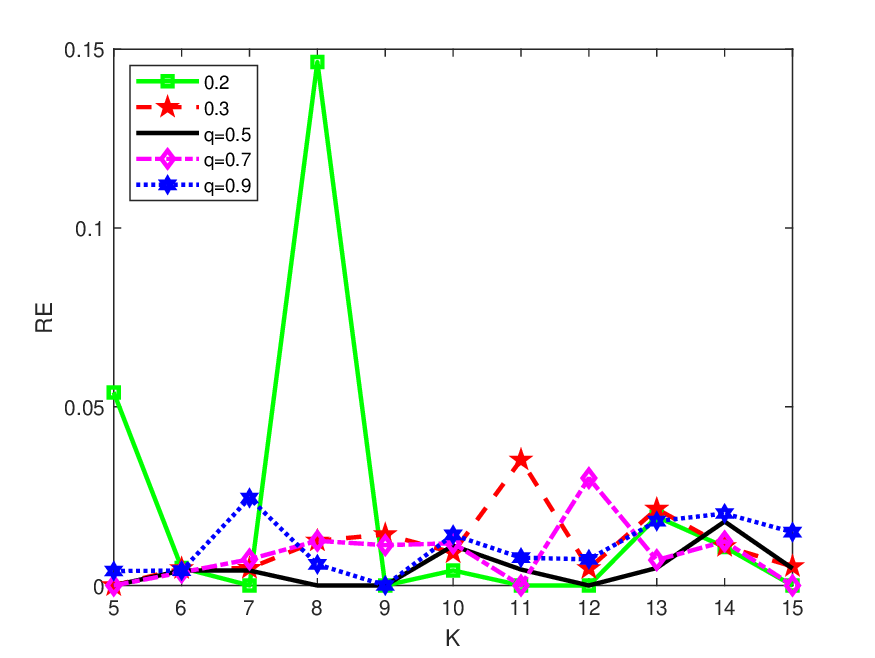}
		\caption{UN (SSN-PMM)}
	\end{subfigure}
	\begin{subfigure}[b]{0.32\textwidth}
		\includegraphics[width=\textwidth]{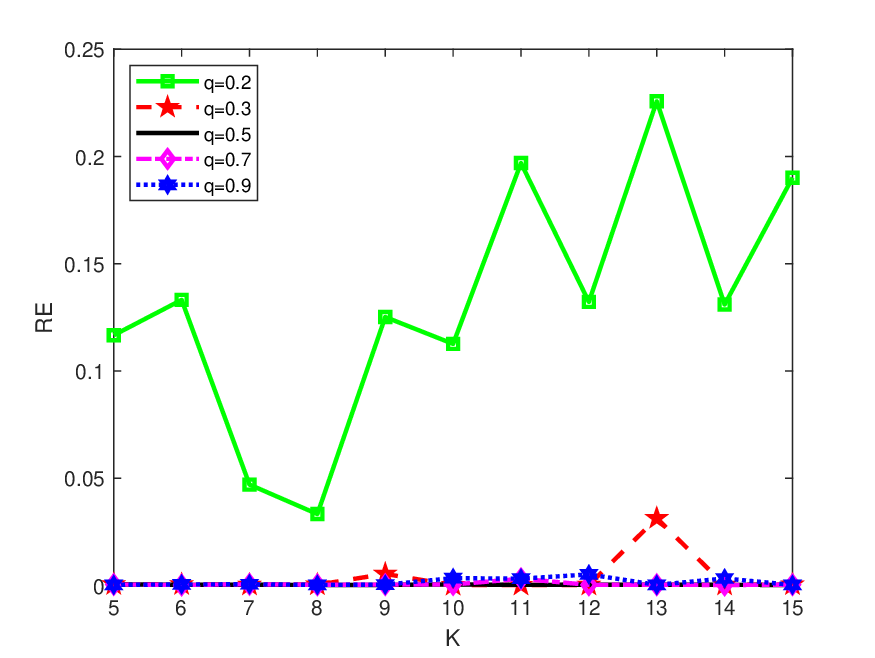}
		\caption{LN (ADMM)}
	\end{subfigure}
	\begin{subfigure}[b]{0.32\textwidth}
		\includegraphics[width=\textwidth]{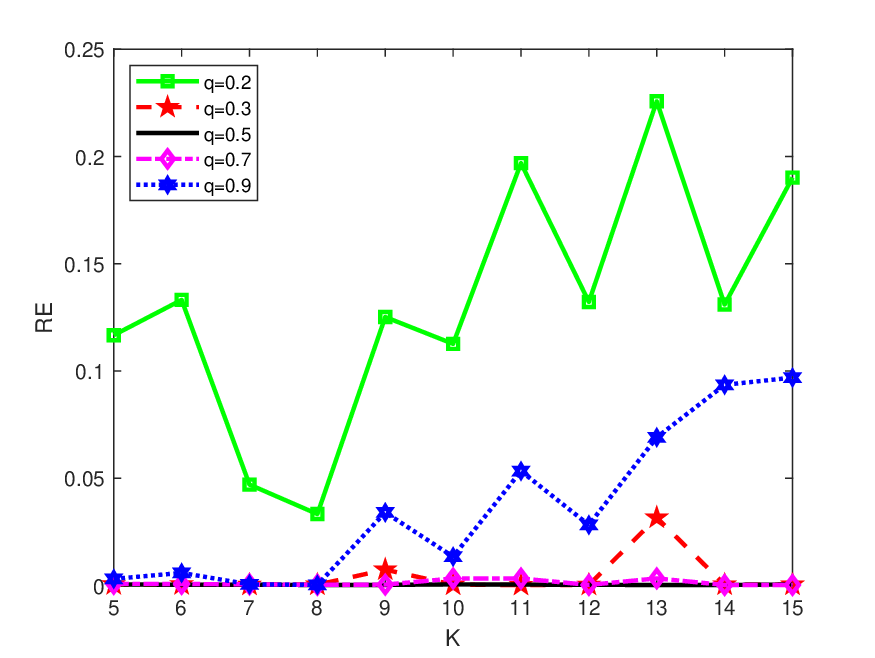}
		\caption{GN (ADMM)}
	\end{subfigure}
	\begin{subfigure}[b]{0.32\textwidth}
		\includegraphics[width=\textwidth]{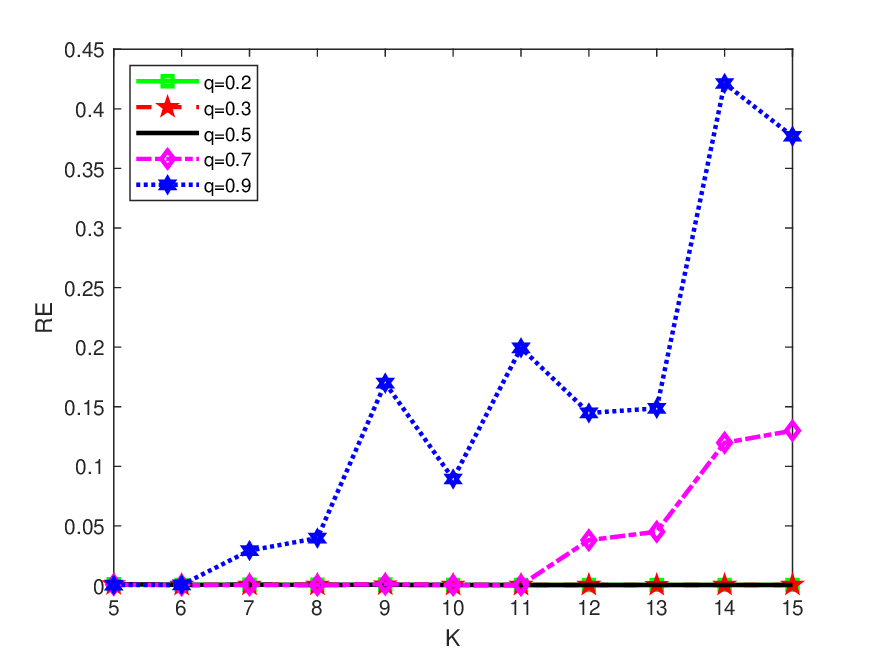}
		\caption{UN (ADMM)}
	\end{subfigure}
	\vspace{-.0cm}
	\caption{\small RE values under different $q$ values and sparsity levels}
	\label{fig2}
	\vspace{-.2cm}
\end{figure}

We focus on the first row of Figure \ref{fig2} which represents the effectiveness of PMM-SSN in terms of relative error (RE) according to different types of noise.
As shown in each plot (a) and (b) that, except for the case of $q=0.2$, the relative error lines remain relatively stable, which indicates that the results are consistent across different sparsity levels.
Additionally, it is evident that the error is significantly larger when $q = 0.2$, while it is much smaller in the other cases.
However, as seen in plot (c), in the UN case, the relative error is significantly small when the sparsity level exceeds $9$,
this suggests that the $\ell_q$-norm in the generalized elastic-net penalty has minimal impact on regression accuracy when the coefficients are sufficiently sparse.
Now we observe the second row of Figure \ref{fig2} to evaluate the performance of ADMM in the sense of the value of RE.
In the cases of LN and GN noise, it is evident that the penalty with $q = 0.2$, the ADMM performs poorly, whereas $q = 0.5$ and $0.7$, ADMM delivers the best performance.
However, in the case of UN noise, ADMM performs best at each sparsity level when $q = 0.3$, but its fails with $q = 0.2$.
Finally, we see from each plot that, the RE remains relatively stable when $q=0.5$ at all sparsity levels.
Therefore, in subsequent experimental tests, we set $q=0.5$ for the proposed generalized elastic-net penalty.

To demonstrate the numerical advantages of ADMM and PMM-SSN, we conducted a series of experiments using different tuning parameter values.
In this test, we set the sparsity level to $K = 10$ and the penalty parameter to $q = 0.5$, as these settings have been shown to be effective.
For simplicity in this test, we fix $\lambda_2$ and gradually vary the value of $\lambda_1$ for the $\ell_q$-norm in the generalized elastic-net penalty of \eqref{lp}.
For  $\lambda_2$, we set $\lambda_2=1e-4$ for both LN and GN noise, while for UN noise, we set $\lambda_2=1e-6$.
In the test, we run PMM-SSN and ADMM $50$ times and record the average values of MSE values (`MSE'), the RE values (`RE') and the numbers of iterations (`Iter')
in Table \ref{tab1}.
From this table we see that  both tested algorithms perform effectively in each case, producing regression solutions with higher accuracy.
Additionally, we observe that the values of RE and MSE obtained using PMM-SSN are generally lower than those from ADMM. This indicates that the second-order SSN employed in the PMM framework is beneficial for achieving higher accuracy solutions.
Focusing on the third column regarding the iterations, we see that PMM-SSN outperforms ADMM, achieving better results in fewer iterations. This further demonstrates that the PMM-SSN algorithm can produce more accurate solutions.

\begin{table}[htbp]
	\centering {\scriptsize\caption{Numerical results of  PMM-SSN and ADMM.}
		\begin{tabular}{cc|ccc|ccc}
			\hline
			\multicolumn{2}{c|}{} & \multicolumn{3}{c|}{PMM-SSN} & \multicolumn{3}{c}{ADMM}\cr
			\hline
			Noise& $\lambda_1$&MSE&RE&Iter&MSE&RE&Iter\cr
			\hline
			\multirow{3}{0.2cm}{LN}
			&$0.15$&$4.272e-08$&$4.204e-05$&$3$&$2.879e-06$&$3.513e-04$&$33$\cr
			&$0.2$&$3.537e-08$&$3.852e-05$&$3$&$1.838e-06$&$2.833e-04$&$36$\cr
			&$0.25$&$1.483e-08$&$2.483e-05$&$3$&$2.255e-06$&$2.946e-04$&$39$\cr
			\hline
			\multirow{3}{0.2cm}{GN}
			&$0.15$&$1.736e-09$&$8.499e-06$&$6$&$2.924e-06$&$3.573e-04$&$34$\cr
			&$0.2$&$2.094e-09$&$9.476e-06$&$9$&$1.377e-06$&$2.429e-04$&$38$\cr
			&$0.25$&$1.729e-06$&$2.645e-04$&$8$&$3.639e-06$&$3.940e-04
			$&$39$\cr
			\hline
			\multirow{3}{0.2cm}{UN}
			&$0.03$&$4.010e-08$&$3.227e-05$&$14$&$3.599e-06$&$3.906e-04$&$120$\cr
			&$0.04$&$2.018e-08$&$2.788e-05$&$9$&$1.292e-05$&$6.954e-04$&$226$\cr
			&$0.05$&$2.974e-10$&$3.604e-06$&$8$&$3.432e-01$&$1.129e-01$&$295$\cr
			\hline
		\end{tabular}\label{tab1}
	}
\end{table}

\subsection{Performance comparisons with IAGENR-Lq}
\begin{sidewaystable}[htbp]
\centering {\scriptsize\caption{Numerical resuts of PMM-SSN, ADMM, and IAGENR-Lq.}
\begin{tabular}{lc|ccc|ccc|ccc}
	\hline
	\multicolumn{2}{c|}{} &
	\multicolumn{3}{c|}{PMM-SSN} &
	\multicolumn{3}{c|}{ADMM}& \multicolumn{3}{c}{IAGENR-Lq}\cr
	\hline
	Matrix(Dim) & K & MSE  & RE & Iter & MSE & RE & Iter & MSE  & RE & Iter\cr
	\hline
	\multirow{3}{0.2cm}{800$\times$400}
	&$5$&$4.308e-07$&$1.660e-04$&$7$&$7.330e-07$&$2.154e-04$&$31$&$2.538e-06$&$4.037e-04$&$7$\cr
	&$7$&$2.157e-07$&$1.118e-04$&$8$&$2.434e-06$&$3.656e-04$&$29$&$2.620e-06$&$3.963e-04$&$7$\cr
	&$9$&$1.211e-06$&$1.807e-04$&$20$&$1.477e-06$&$2.822e-04$&$40$&$3.220e-06$&$4.036e-04$&$7$\cr
	\hline
	\multirow{3}{0.2cm}{1000$\times$300}
	&$5$&$4.252e-07$&$9.523e-05
	$&$8$&$1.053e-05$&$6.560e-04$&$36$&$1.177e-05$&$7.062e-04$&$9$\cr
	&$10$&$2.094e-09$&$9.476e-06$&$8$&$1.377e-06$&$2.429e-04$&$37$&$1.113e-05$&$6.962e-04$&$9$\cr
	&$15$&$1.069e-08$&$1.829e-05$&$10$&$3.207e-06$&$3.648e-04$&$40$&$1.071e-05$&$6.795e-04$&$10$\cr
	\hline
	\multirow{3}{0.2cm}{2000$\times$500}
	&$6$&$5.531e-07$&$3.045e-04$&$8$&$1.835e-06$&$5.491e-04$&$36$&$4.964e-06$&$8.922e-04$&$8$\cr
	&$8$&$8.229e-07$&$3.631e-04$&$9$&$1.346e-06$&$4.659e-04$&$38$&$5.392e-06$&$9.186e-04$&$8$\cr
	&$10$&$7.9452e-07$&$2.4117e-04$&$9$&$1.543e-06$&$4.158e-04$&$38$&$6.004e-06$&$ 8.538e-04$&$10$\cr
	\hline
\end{tabular}\label{tab2}
}
\end{sidewaystable}
In this section, we  assess the superiority of the proposed model \eqref{lp} and evaluage the efficiency of PMM-SSN and ADMM by comparing them to the state-of-the-art algorithm IAGENR-Lq developed by Zhang et al. \cite{lqnet}.
The data utilized in this part is identical to that in  Section \eqref{set41}, but here we take the correlation coefficient $\kappa$ as $0.3$.
It is important to note that the proposed model \eqref{lp} is novel, and no existing algorithms can be employed for comparison. Therefore, we test it against the least squares estimation model \eqref{net1} to evaluate the superiority of \eqref{lp}.
Additionally, in this comparison test, we specifically focus on $r = 2$, indicating that both models being evaluated are designed to handle Gaussian noise.
As evident by the previous tests, we take $q=0.5$ in the generalized elastic-net penalty, and choose $\lambda_1=0.2$ and $\lambda_2=1e-4$  for the tuning parameters.
Furthermore, we select $\lambda_1 = 1e-8$ and $\lambda_2 = 1e-4$ for IAGENR-Lq, as these values have consistently demonstrated their effectiveness in experiments' preparations.
In this test, we consider various values of $n$ and $ p$, along with different levels of sparsity $ K$. The results for MSE, RE, and Iter are presented in Table \ref{tab2}.

Upon analyzing the first two columns of Table \ref{tab2}, which present the MSE and RE values for each algorithm, we observe that PMM-SSN achieves the lowest values, while IAGENR-Lq records the highest.
We see that PMM-SSN and ADMM slightly outperform the IAGENR-Lq in terms of regression accuracy.
In summary, these results support the conclusion that our proposed model  \eqref{lp} improve the quality of regression quality derived from model \eqref{net1}.

\subsection{Test on real data}  
In this section, we conduct numerical comparisons using the ``pyrim" and ``bodyfat" instances from the LIBSVM dataset, which are commonly utilized for regression purposes.
These datasets can be obtained from \href{https://www.csie.ntu.edu.tw/~cjlin/libsvmtools/datasets/binary.html}
{https://www.csie.ntu.edu.tw/~cjlin/libsvmtools/datasets/binary.html}.
To address the high-dimensional settings, we use the approach described in \cite{polynomial}, which expands the original data features with polynomial basis functions.
For example, the suffix ``3" in data ``pyrim3" indicates that a 3rd degree polynomial is used to generate the basis function.
In this real data test, we terminate the iterative process when   ``$\eta_1 \leq 1e -4$" is achieved. 
The computational results of PMM-SSN and ADMM under different noise interferences are listed in Table \ref{tab3}. 
Since the true sparse coefficients are unknown for these real data examples,  the results of RE and MSE are not included in this table.
In PMM-SSN, we use the fixed parameters' values: $\mu=1e-2$, $\rho=0.9$, $\varrho=0.1$, $\delta=0.5$, and $\nu=1e-6$, while in ADMM,   we take $\sigma=1e-3$ and $\tau=1.618$.
To evaluate performance of both algorithms, we utilize three measurements: the final convergence value $\eta_1$ (RES), the number of iterations (Iter), and CPU computation time (Time) in seconds. 

\begin{table}[htbp]
	\centering {\footnotesize\caption{Numerical results of  PMM-SSN and ADMM on real data.}
		\begin{tabular}{lccc|ccc|ccc}
			\hline
			\multicolumn{4}{c|}{} &
			\multicolumn{3}{c|}{PMM-SSN} &
			\multicolumn{3}{c}{ADMM}\cr
			\hline
			Data(Dim) & Noise & $\lambda_1$ & $\lambda_2$ & Iter & RES & Time & Iter & RES & Time\cr
			\hline
			\multirow{3}{0.2cm}{pyrim3
				74;3654}
			&LN& $0.3$&$1e-2$&$4$ & $9.050e-02$ & $0.2309$ & $36$ & $1.277e-01$ & $5.0322$\cr
			&GN& $0.3$&$1e-2$&$20$ & $1.109e-01$ & $1.0018$ & $36$ & $1.277e-01$ & $4.9444$\cr
			&UN& $0.1$&$1e-1$&$80$ & $6.950e-02$ & $4.7067$ & $162$ & $1.043e-01$ & $18.0183$\cr
			\hline
			\multirow{3}{0.2cm}{bodyfat5
				252;11628}
			&LN& $0.2$&$2e-3$&$3$ & $1.720e-02$ & $1.6133$ & $15$ & $2.842e-01$ & $24.5491$\cr
			&GN& $0.3$&$1e-2$&$6$ & $4.070e-02$ & $3.0155$ & $48$ & $1.816e-01$ & $78.6079$\cr
			&UN& $0.2$&$1e-3$&$7$ & $2.230e-02$ & $3.6474$ & $11$ & $2.769e-01$ & $20.0615$\cr
			\hline
		\end{tabular}\label{tab3}
	}
\end{table}

Based on the data presented, we can analyze the efficiency of the PMM-SSN and ADMM.
We see that PMM-SSN shows better performance across various parameter settings. For instance, in the LN case with parameters set to $\lambda_1=0.3$ and $\lambda_2=1e-2$, the RES value derived by PMM-SSN is $9.050e-02$, but the one by ADMM is only $1.277e-1$.
In the most cases, PMM-SSN consistently demonstrates better efficiency, making it a preferable choice for the task of sparse linear regression.
The ``Time" column shows that PMM-SSN generally has shorter computing times compared to ADMM. 
For instance, in both the LN and GN cases, PMM runs significantly faster. In contrast, ADMM takes longer computing time, especially with higher dimensional data. Overall, PMM-SSN's efficiency in computing time once again makes it the preferred choice for sparse linear regression.

\section{Conclusions}\label{con5}

In this paper, we proposed a generalized elastic-net model for sparse linear regression that utilizes a $\ell_r$-norm as the loss function and incorporates a $\ell_q$-norm in the elastic-net penalty.
We proved that the local minimizer of the proposed model is  a generalized first-order stationary point (Thorem \ref{pro1}), and then derived the lower bound for the nonzero entries of this generalized first-order stationary point (Theorem \ref{the1}). 
Furthermore, utilizing an $\epsilon$-approximation function $h_{u_{\epsilon}}(\cdot)$ from Lu \cite{lu}, we proposed an $\epsilon$-approximation model and proved that if the $\epsilon>0$ is chosen properly, the local minimizer of the proposed model is a also a generalized first-order stationary point of this $\epsilon$-approximation model (Theorem \ref{the2}).

For practical implementation, we proposed two efficient algorithms within an iterative reweighted framework to solve the resulting $\epsilon$-approximation model: one is the easily implementable first-order algorithm ADMM, and the other is the faster PMM-SSN algorithm, which leverages second-order information.
To demonstrate the effectiveness of the proposed algorithms, we conducted a series of numerical experiments using both simulated data and a real dataset. We compared their performance against the state-of-the-art algorithm, IAGENR-Lq. The results indicated that both algorithms are effective in addressing high-dimensional sparse linear regression problems and exhibited strong robustness across different types of noise.
Our findings revealed that while ADMM performed well in numerical experiments due to its straightforward implementation, the PMM-SSN demonstrated superior performance in tackling complex problems, particularly when second-order information plays a crucial role in improving solution accuracy.
Consequently, for high-dimensional sparse linear regression problems that demand higher precision, the PMM-SSN may offer a more effective choice.

\bmhead{Acknowledgements}

The work of Yanyun Ding is supported by the Shenzhen Polytechnic University Research Fund (Grant No. 6024310021K). The work of Peili Li is supported by the National Natural Science Foundation of China (Grant No. 12301420).
The work of Yunhai Xiao is supported by the National Natural Science Foundation of China (Grant No. 12471307 and 12271217), and the National Natural Science Foundation of Henan Province (Grant No. 232300421018).

\section*{Conflict of interest statement}

The authors declare that they have no conflict of interest.

\bibliography{hcfs}

\end{document}